\documentclass{article}

\usepackage{ifthen}
\newboolean{arxivVersion}
\setboolean{arxivVersion}{true} 

\ifthenelse{\boolean{arxivVersion}}{\usepackage[margin=1in]{geometry}
\usepackage{environ}
\newcommand{\acksection}{\section*{Acknowledgments and Disclosure of Funding}}
\NewEnviron{ack}{%
  \acksection
  \BODY
}}
{}

\usepackage[utf8]{inputenc} 
\usepackage[T1]{fontenc}    
\usepackage{hyperref}       
\usepackage{url}            
\usepackage{booktabs}       
\usepackage{amsmath}
\usepackage{color}
\usepackage{graphicx}
\usepackage{multirow}
\usepackage{amsfonts}       
\usepackage{nicefrac}       
\usepackage{microtype}      
\usepackage{wrapfig}
\usepackage{algorithm}
\usepackage{algorithmic}
\usepackage{xfrac}
\usepackage{caption}
\usepackage{xcolor}
\hypersetup{
	colorlinks,
	linkcolor={red!40!gray},
	citecolor={blue!40!gray},
	urlcolor={blue!70!gray}
}

\def\simiid{\overset{\textnormal{\fontsize{6}{6}\selectfont
i.i.d.}}{\sim}}

\def\IID{{\mathrm{IID}}}
\def\nnz{{\mathrm{nnz}}}

\def\polylog{{\mathrm{polylog}}}
\def\tinydots{\textnormal{\fontsize{6}{6}\selectfont \dots}}

\def\Det{{\mathrm{Det}}}

\DeclareMathOperator{\adj}{\mathrm{adj}}

\def\xbh{\hat{\x}}

\def\g {\mathbf{g}}

\def\ee{\mathrm{e}}

\def\Nc{\mathcal{N}}

\def\p{\mathbf p}

\def\Y{\mathbf Y}
\def\R{\mathbf R}
\def\H{\mathbf H}
\def\Hbh{\widehat{\H}}
\def\Hbt{\widetilde{\H}}

\ifx\BlackBox\undefined
\newcommand{\BlackBox}{\rule{1.5ex}{1.5ex}}  
\fi

\def\DPP{{\mathrm{DPP}}}
\DeclareMathOperator*{\argmin}{\mathop{\mathrm{argmin}}}

\def\x{\mathbf x}

\def\y{\mathbf y}

\def\a{\mathbf a}
\def\b{\mathbf b}

\def\v{\mathbf v}

\def\p{\mathbf{p}}

\def\e{\mathbf e}
\def\zero{\mathbf 0}
\def\one{\mathbf 1}

\def\Poisson{\mathrm{Poisson}}
\def\X{\mathbf X}

\def\Xb{\bar{\X}}
\def\Xc{\mathcal{X}}

\def\A{\mathbf A}
\def\C{\mathbf C}

\def\M{\mathbf M}

\def\S{\mathbf S}
\def\Sb{\bar\S}

\def\St{\widetilde{S}}

\def\Sc{\mathcal{S}}
\def\Z{\mathbf Z}

\def\I{\mathbf I}

\def\A{\mathbf A}

\def\E{\mathbb E}

\def\R{\mathbb R} 
\def\tr{\mathrm{tr}}

\let\origtop\top
\renewcommand\top{{\scriptscriptstyle{\origtop}}} 

\definecolor{silver}{cmyk}{0,0,0,0.3}
\definecolor{yellow}{cmyk}{0,0,0.9,0.0}
\definecolor{reddishyellow}{cmyk}{0,0.22,1.0,0.0}
\definecolor{black}{cmyk}{0,0,0.0,1.0}
\definecolor{darkYellow}{cmyk}{0.2,0.4,1.0,0}
\definecolor{darkSilver}{cmyk}{0,0,0,0.1}
\definecolor{grey}{cmyk}{0,0,0,0.5}
\definecolor{darkgreen}{cmyk}{0.6,0,0.8,0}

\newcommand{\Green}[1]{{\color{darkgreen}  {#1}}}
\newcommand{\Blue}[1]{\color{blue}{#1}\color{black}}
\newcommand{\Brown}[1]{{\color{brown}{#1}\color{black}}}

\newenvironment{proofof}[2]{\par\vspace{2mm}\noindent\textbf{Proof of {#1} {#2}}\ }{\hfill\BlackBox}

\ifx\proof\undefined
\newenvironment{proof}{\par\noindent{\bf Proof\ }}{\hfill\BlackBox\\[2mm]}
\fi

\ifx\theorem\undefined
\newtheorem{theorem}{Theorem}
\fi

\ifx\example\undefined
\newtheorem{example}{Example}
\fi

\ifx\property\undefined
\newtheorem{property}{Property}
\fi

\ifx\lemma\undefined
\newtheorem{lemma}[theorem]{Lemma}
\fi

\ifx\proposition\undefined
\newtheorem{proposition}[theorem]{Proposition}
\fi

\ifx\remark\undefined
\newtheorem{remark}[theorem]{Remark}
\fi

\ifx\corollary\undefined
\newtheorem{corollary}[theorem]{Corollary}
\fi

\ifx\definition\undefined
\newtheorem{definition}{Definition}
\fi

\ifx\conjecture\undefined
\newtheorem{conjecture}[theorem]{Conjecture}
\fi

\ifx\axiom\undefined
\newtheorem{axiom}[theorem]{Axiom}
\fi

\ifx\claim\undefined
\newtheorem{claim}[theorem]{Claim}
\fi

\ifx\assumption\undefined
\newtheorem{assumption}[theorem]{Assumption}
\fi

\definecolor{lasallegreen}{rgb}{0.03, 0.47, 0.19}

\title{Debiasing Distributed Second Order Optimization\\
with Surrogate Sketching and Scaled Regularization}

\author{%
  \textbf{Micha{\l } Derezi\'{n}ski} \\
  Department of Statistics\\
  University of California, Berkeley\\
  \texttt{mderezin@berkeley.edu} \\
   \and
  \textbf{Burak Bartan} \\
  Department of Electrical Engineering\\
  Stanford University\\
  \texttt{bbartan@stanford.edu} \\
   \and
  \textbf{Mert Pilanci} \\
  Department of Electrical Engineering\\
  Stanford University\\
  \texttt{pilanci@stanford.edu} \\
   \and
  \textbf{Michael W. Mahoney} \\
  ICSI and Department of Statistics\\
  University of California, Berkeley\\
  \texttt{mmahoney@stat.berkeley.edu} \\
}

\begin{document}
\maketitle

\begin{abstract}
In distributed second order optimization, a standard strategy is to average many local estimates, each of which is based on a small sketch or batch of the data. 
However, the local estimates on each machine are typically biased, relative to the full solution on all of the data, and this can limit the effectiveness of averaging.  
Here, we introduce a new technique for debiasing the local estimates, which leads to both theoretical and empirical improvements in the convergence rate of distributed second order methods. Our technique has two novel components: (1) modifying standard sketching techniques to obtain what we call a surrogate sketch; and (2) carefully scaling the global regularization parameter for local computations. Our surrogate sketches are based on determinantal point processes, a family of distributions for which the bias of an estimate of the inverse Hessian can be computed exactly. Based on this computation, we show that when the objective being minimized is $l_2$-regularized with parameter $\lambda$ and individual machines are each given a sketch of size $m$, then to eliminate the bias, local estimates should be computed using a shrunk regularization parameter given by $\lambda^{\prime}=\lambda\cdot(1-\frac{d_{\lambda}}m)$, where $d_{\lambda}$ is the $\lambda$-effective
dimension of the Hessian (or, for quadratic problems, the data matrix). 
\end{abstract}
\section{Introduction}
\label{sec:intro}

We consider the task of second order optimization in a distributed or parallel setting. Suppose that $q$ workers are each given a small sketch of the data (e.g., a random sample or a random projection) and a parameter vector $\x_t$. The goal of the $k$-th worker is to construct a local estimate $\Delta_{t,k}$ of the Newton step relative to a convex loss on the full dataset. The estimates are then averaged and the parameter vector is updated using this averaged step, obtaining $\x_{t+1} = \x_t + \frac1q\sum_{k=1}^q \Delta_{t,k}$. This basic strategy has been extensively studied and it has proven effective for a variety of optimization tasks because of its communication-efficiency \cite{mahoney2018giant}. However, a key problem that limits the scalability of this approach is that local estimates of second order steps are typically biased, which means that for a sufficiently large $q$, adding more workers will not lead to any improvement in the convergence rate. 
Furthermore, for most types of sketched estimates this bias is difficult to compute, or even approximate, which makes it difficult to correct.

In this paper, we propose a new class of sketching methods, called \emph{surrogate sketches}, which allow us to debias local estimates of the Newton step, thereby making distributed second order optimization more scalable. 
In our analysis of the surrogate sketches, we exploit recent developments in determinantal point processes (DPPs) to give exact formulas for the bias of the estimates produced with those sketches, enabling us to correct that bias. 
Due to algorithmic advances in DPP sampling, surrogate sketches can be implemented in time nearly linear in the size of the data, when the number of data points is much larger than their dimensionality, so our results lead to direct improvements in the time complexity of distributed second order optimization. 
Remarkably, our analysis of the bias of surrogate sketches leads to a simple technique for debiasing the local Newton estimates for $l_2$-regularized problems, which we call \emph{scaled regularization}. We show that the regularizer used on the sketched data should be scaled down compared to the global regularizer, and we give an explicit formula for that scaling. Our empirical results demonstrate that scaled regularization significantly reduces the bias of local Newton estimates not only for surrogate sketches, but also for a range of other sketching techniques.

\subsection{Debiasing via Surrogate Sketches and Scaled Regularization}
\label{sec:intro-debiasing}

Our scaled regularization technique applies to sketching the Newton step over a convex loss, as described in Section \ref{sec:unbiased}, however, for concreteness, we describe it here in the context of regularized least squares.
Suppose that the data is given in the form of an $n\times d$ matrix $\A$ and an $n$-dimensional vector $\b$, where $n\gg d$. For a given regularization parameter $\lambda>0$, our goal is to approximately solve the following problem: 
\begin{align}
  \x^* = \argmin_\x\,\frac12\|\A\x-\b\|^2 + \frac\lambda2\|\x\|^2.\label{eq:least-squares}
\end{align}
Following the classical sketch-and-solve paradigm, we use a random $m\times n$ sketching matrix $\S$, where $m\ll n$, to replace this large $n\times d$ regularized least squares problem $(\A,\b,\lambda)$ with a smaller $m\times d$ problem of the same form. We do this by sketching both the matrix $\A$ and the vector $\b$, obtaining the problem $(\S\A,\S\b,\lambda^{\prime})$ given by:
\begin{align}
  \xbh = \argmin_\x\,\frac12\|\S\A\x-\S\b\|^2 + \frac{\lambda^{\prime}}{2}\|\x\|^2,\label{eq:sketch-and-solve}
\end{align}
where we deliberately allow $\lambda^{\prime}$ to be different than $\lambda$. The question we pose is: What is the right choice of $\lambda^{\prime}$ so as to minimize $\|\E[\xbh]-\x^*\|$, i.e., the bias of $\xbh$, which will dominate the estimation error in the case of massively parallel averaging? We show that the choice of $\lambda^{\prime}$ is controlled by a classical notion of effective dimension for regularized least squares \cite{ridge-leverage-scores}.
\begin{definition}
Given a matrix $\A$ and regularization parameter $\lambda\geq 0$, the
  $\lambda$-effective dimension of $\A$ is defined as
  $d_{\lambda}=d_\lambda(\A) = \tr(\A^\top\A(\A^\top\A+\lambda\I)^{-1})$.
\end{definition}
For surrogate sketches, which we define in Section \ref{s:surrogate}, it is in fact possible to bring the bias down to zero and we give an exact formula for the correct $\lambda^{\prime}$ that achieves this (see Theorem~\ref{t:least-squares} in Section~\ref{sec:unbiased} for a statement which applies more generally to the Newton's method).
\begin{theorem}\label{t:intro-unbiased}
    If $\xbh$ is constructed using a size $m$ surrogate sketch from Definition \ref{d:surrogate}, then:
    \[\E[\xbh] = \x^*\quad\text{for}\quad \lambda^{\prime} = \lambda\cdot\Big(1 - \frac{d_\lambda}{m}\Big).\] 
\end{theorem}

\ifthenelse{\boolean{arxivVersion}}{
\begin{figure}
\centering
\centerline{\includegraphics[width=0.4\textwidth]{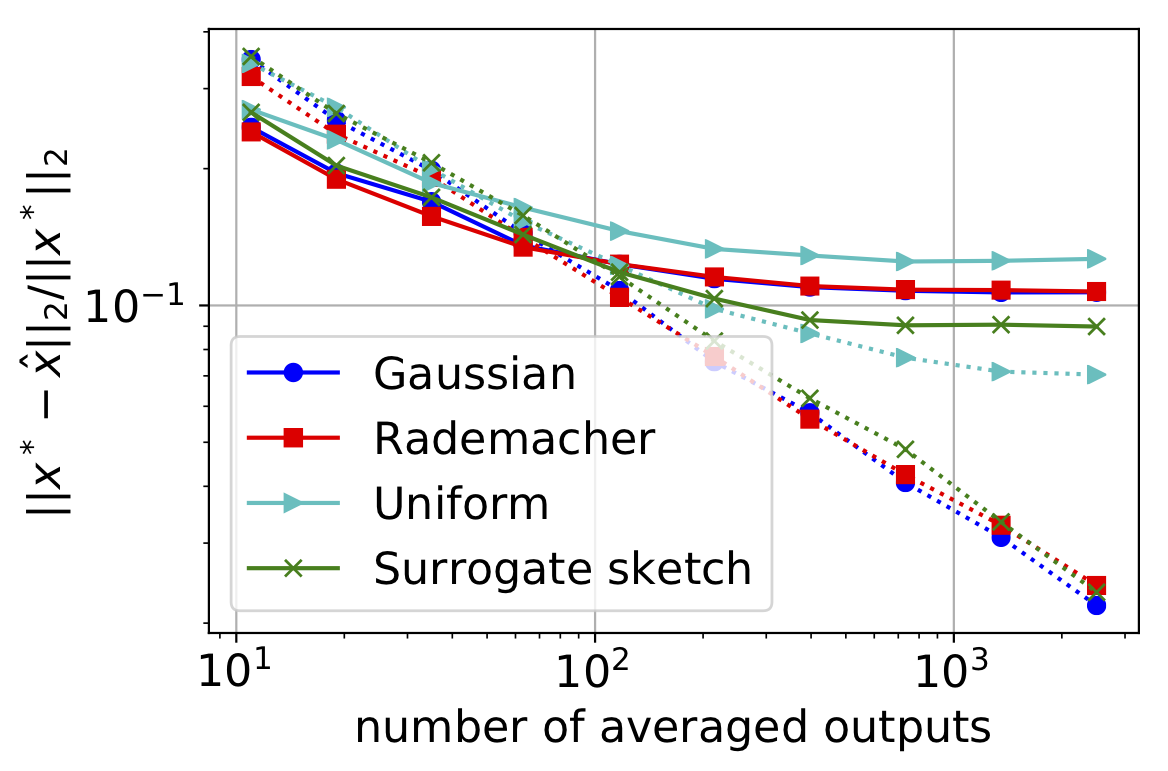}}
\caption{Estimation error against the number of averaged outputs for the Boston housing prices dataset (see Section \ref{sec:numerical}). The dotted curves show the error when the regularization parameter is rescaled as in Theorem \ref{t:intro-unbiased}.}
\label{fig:intro-debiasing}
\end{figure}
}{
\begin{wrapfigure}{r}{0.4\textwidth}
\centering
\vspace{-5mm}
\centerline{\includegraphics[width=0.38\textwidth]{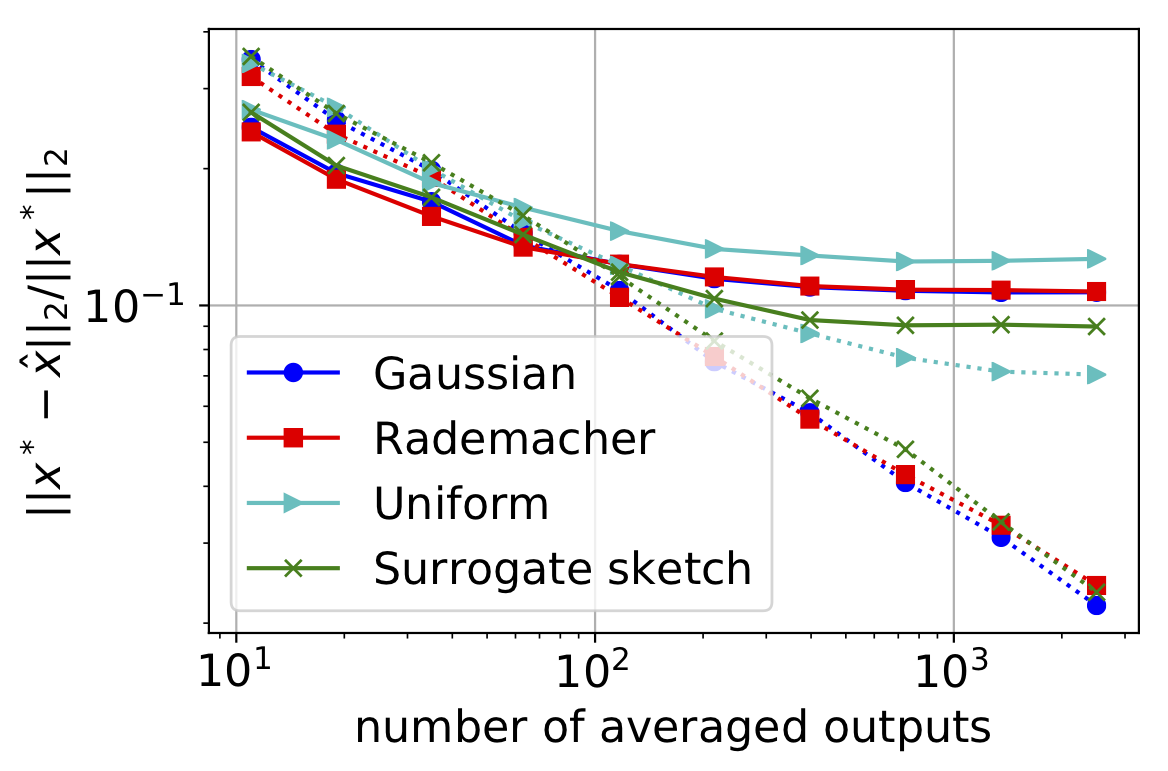}}
\vspace{-2mm}
\caption{Estimation error against the number of averaged outputs for the Boston housing prices dataset (see Section \ref{sec:numerical}). The dotted curves show the error when the regularization parameter is rescaled as in Theorem \ref{t:intro-unbiased}.}
\vspace{-5mm}
\label{fig:intro-debiasing}
\end{wrapfigure}
}

Thus, the regularization parameter used to compute the local estimates should be smaller than the global regularizer $\lambda$. While somewhat surprising, this observation does align with some prior empirical \cite{sketched-ridge-regression} and theoretical \cite{surrogate-design} results which suggest that random sketching or sampling introduces some amount of \emph{implicit regularization}. From this point of view, it makes sense that we should compensate for this implicit effect by reducing the amount of \emph{explicit} regularization being used. 

One might assume that the above formula for $\lambda^{\prime}$ is a unique property of surrogate sketches. 
 However, we empirically show that our scaled regularization applies much more broadly, by testing it with the standard Gaussian sketch ($\S$ has i.i.d.~entries $\Nc(0,1/m)$), a Rademacher sketch ($\S$ has i.i.d.~entries equal $\frac1{\sqrt m}$ or $-\frac1{\sqrt m}$ with probability $0.5$), and uniform row sampling. In Figure \ref{fig:intro-debiasing}, we plot normalized estimates of the bias, $\|(\frac1q\sum_{k=1}^q\xbh_k)-\x^*\|/\|\x^*\|$, by averaging $q$ i.i.d.~copies of $\xbh$, as $q$ grows to infinity, showing the results with both scaled (dotted curves) and un-scaled (solid curves) regularization. Remarkably, the scaled regularization seems to correct the bias of $\xbh$ very effectively for Gaussian and Rademacher sketches as well as for the surrogate sketch, resulting in the estimation error decaying to zero as $q$ grows. For uniform sampling, scaled regularization also noticeably reduces the bias. In Section \ref{sec:numerical} we present experiments on more datasets which further verify these claims.

\begin{table}
\centering
\begin{tabular}{r|lll|ll}
&Sketch& Averaging& Regularizer  & Convergence Rate & Assumption\\
  \hline\hline
  &&&&\\
\cite{mahoney2018giant}&i.i.d.~row sample &uniform&$\lambda$
  &$\big(\frac1{\alpha q}+\frac1{\alpha^2}\big)^t$
    &$\alpha\geq 1$
  \\  &&&& \\
\cite{determinantal-averaging}&i.i.d.~row sample &determinantal&$\lambda$
  &$\big(\frac d{\alpha q}\big)^t$
    &$\alpha\geq d$
  \\  &&&&\\
Thm.~\ref{t:main} &surrogate sketch &uniform&$\lambda\cdot\big(1-\frac{d_{\lambda}}{m}\big)$
  &$\big(\frac1{\alpha q}\big)^t$
& $\alpha\geq d$
  \end{tabular}
  \vspace{3mm}
    \caption{Comparison of convergence guarantees for the Distributed Iterative Hessian Sketch on regularized least squares (see Theorem \ref{t:main}), with $q$ workers and sketch size $m=\tilde O(\alpha\,d)$. Note that both the references \cite{mahoney2018giant,determinantal-averaging} state their results for uniform sampling sketches.  This can be easily adapted to leverage score sampling, in which case each sketch costs $\tilde O(\nnz(\A) + \alpha d^3)$ to construct. }
    \vspace{-3mm}
    \label{tab:convergence}
\end{table}

\subsection{Convergence Guarantees for Distributed Newton Method}
\label{sec:intro-theory}

We use the debiasing technique introduced in Section \ref{sec:intro-debiasing} to obtain the main technical result of this paper, which gives a convergence and time complexity guarantee for distributed Newton's method with surrogate sketching. 
Once again, for concreteness, we present the result here for the regularized least squares problem \eqref{eq:least-squares}, but a general version for convex losses is given in Section \ref{sec:convergence_analysis} (see Theorem \ref{t:newton}). Our goal is to perform a distributed and sketched version of the classical Newton step: $\x_{t+1}=\x_t-\H^{-1}\g(\x_t)$, where $\H=\A^\top\A+\lambda\I$ is the Hessian of the quadratic loss, and $\g(\x_t)=\A^\top(\A\x_t-\b) + \lambda\x_t$ is the gradient.
To efficiently approximate this step, while avoiding the $O(nd^2)$ cost of computing the exact Hessian, we use a distributed version of the so-called Iterative Hessian Sketch (IHS), which replaces the Hessian with a sketched version $\Hbh$, but keeps the exact gradient, resulting in the update direction $\Hbh^{-1}\g(\x_t)$ \cite{PilWai14b,ozaslan2019regularized,lacotte2020optimal,lacotte2019faster}. 
Our goal is that $\Hbh$ should be cheap to construct and it should lead to an unbiased estimate of the exact Newton step $\H^{-1}\g$. When the matrix $\A$ is sparse, it is desirable for the algorithm to run in time that depends on the input sparsity, i.e., the number of non-zeros denoted $\nnz(\A)$.
\begin{theorem}\label{t:main}
Let $\kappa$ denote the condition number of the Hessian $\H$, let $\x_0$ be the initial parameter vector and take any $\alpha\geq d$. There is an algorithm which returns a Hessian sketch $\Hbh$ in time $O(\nnz(\A)\log(n) + \alpha d^3\,\polylog(n,\kappa,1/\delta))$, such that if $\Hbh_1,...,\Hbh_q$ are i.i.d.~copies of $\Hbh$ then,
\begin{align*}
    \text{Distributed IHS:}\qquad \x_{t+1} = \x_t - \frac1q\sum_{k=1}^q\Hbh_k^{-1}\g(\x_t), 
\end{align*}
with probability $1-\delta$ enjoys a linear convergence rate given as follows:
\begin{align*}
    \|\x_t-\x^*\|^2 \leq \rho^t\kappa\,\|\x_0-\x^*\|^2,\qquad\text{where}\quad\rho=\frac1{\alpha q}.
\end{align*}
\end{theorem}
\begin{remark}
    To reach $\|\x_t-\x^*\|^2\leq\epsilon\cdot\|\x_0-\x^*\|^2$ we need $t\leq\frac{\log(\kappa/\epsilon)}{\log(\alpha q)}$ iterations. See Theorem \ref{t:newton} in Section \ref{sec:convergence_analysis} for a general result on convex losses of the form $f(\x) = \frac1n\sum_{i=1}^n\ell_i(\x^\top\varphi_i) + \frac\lambda2\|\x\|^2$.
\end{remark}
Crucially, the linear convergence rate $\rho$ decays to zero as $q$ goes to infinity, which is possible because the local estimates of the Newton step produced by the surrogate sketch are unbiased. Just like commonly used sketching techniques, our surrogate sketch can be interpreted as replacing the matrix $\A$ with a smaller matrix $\S\A$, where $\S$ is a $m\times n$ sketching matrix, with $m=\tilde O(\alpha d)$ denoting the sketch size. Unlike the Gaussian and Rademacher sketches, the sketch we use is very sparse, since it is designed to only sample and rescale a subset of rows from $\A$, which makes the multiplication very fast. 
Our surrogate sketch has two components: (1) standard i.i.d.~row sampling according to the so-called $\lambda$-ridge leverage scores \cite{drineas2006sampling,ridge-leverage-scores}; and (2) non-i.i.d.~row sampling according to a determinantal point process (DPP) \cite{dpp-ml}.  While leverage score sampling has been used extensively as a sketching technique for second order methods, it typically leads to biased estimates, so combining it with a DPP is crucial to obtain strong convergence guarantees in the distributed setting. The primary computational costs in constructing the sketch come from estimating the leverage scores and sampling from the DPP.

\subsection{Related Work}\label{sec:related-work}

While there is extensive literature on distributed second order methods, it is useful to first compare to the most directly related approaches. In Table \ref{tab:convergence}, we contrast Theorem \ref{t:main} with two other results which also analyze variants of the Distributed IHS, with all sketch sizes fixed to $m=\tilde O(\alpha d)$. The algorithm of \cite{mahoney2018giant} simply uses an i.i.d.~row sampling sketch to approximate the Hessian, and then uniformly averages the estimates. This leads to a bias term $\frac1{\alpha^2}$ in the convergence rate, which can only be reduced by increasing the sketch size. In \cite{determinantal-averaging}, this is avoided by performing weighted averaging, instead of uniform, so that the rate decays to zero with increasing $q$. Similarly as in our work, determinants play a crucial role in correcting the bias, however with significantly different trade-offs. While they avoid having to alter the sketching method, the weighted average introduces a significant amount of variance, which manifests itself through the additional factor $d$ in the term $\frac{d}{\alpha q}$. Our surrogate sketch avoids the additional variance factor while maintaining the scalability in $q$. The only trade-off is that the time complexity of the surrogate sketch has a slightly worse polynomial dependence on $d$, and as a result we require the sketch size to be at least $\tilde O(d^2)$, i.e., that $\alpha\geq d$. Finally, unlike the other approaches, our method uses a scaled regularization parameter to debias the Newton estimates.

Distributed second order optimization has been considered by many other works in the literature and many methods have been proposed such as DANE \cite{dane}, AIDE \cite{aide}, DiSCO \cite{disco}, and others \cite{MokhtariLR17,BajovicJKJ17}. Distributed averaging has been discussed in the context of linear regression problems in works such as \cite{bartan2020distributed_privacy} and studied for ridge regression in \cite{sketched-ridge-regression}. However, unlike our approach, all of these methods suffer from biased local estimates for regularized problems. Our work deals with distributed versions of iterative Hessian sketch and Newton sketch and convergence guarantees for non-distributed version are given in \cite{PilWai14b} and \cite{pilanci2017newton}. Sketching for constrained and regularized convex programs and minimax optimality has been studied in \cite{PilWai14a,yang2017randomized,sridhar2020lower}. Optimal iterative sketching algorithms for least squares problems were investigated in \cite{lacotte2020limiting,lacotte2020optimal,lacotte2019faster,lacotte2020effective,lacotte2019high}.
Bias in distributed averaging has been recently considered in \cite{bartan2020distributed}, which provides expressions for regularization parameters for Gaussian sketches. The theoretical analysis of \cite{bartan2020distributed} assumes identical singular values for the data matrix whereas our results make no such assumption. Finally, our analysis of surrogate sketches builds upon a recent line of works which derive expectation formulas for determinantal point processes in the context of least squares regression \cite{unbiased-estimates-journal,correcting-bias,minimax-experimental-design,surrogate-design}.

\section{Surrogate Sketches}
\label{s:surrogate}

In this section, to motivate our surrogate sketches, we consider several standard sketching techniques and discuss their shortcomings. 
Our purpose in introducing surrogate sketches is to enable exact analysis of the sketching bias in second order optimization, thereby permitting us to find the optimal hyper-parameters for distributed averaging.


Given an $n\times d$ data matrix $\A$, we define a standard
sketch of $\A$ as the matrix $\S\A$, where $\S\sim \Sc_\mu^m$ is a random
$m\times n$ matrix with $m$ i.i.d.~rows distributed according to measure $\mu$
with identity covariance, rescaled so that $\E[\S^\top\S]=\I$. This
includes such standard sketches as:
\begin{enumerate}
  \item \emph{Gaussian sketch:} each row of $\S$ is distributed as
    $\Nc(\zero,\frac1m\I)$.
    \item \emph{Rademacher sketch:} each entry of $\S$ is $\frac1{\sqrt m}$ with probability
      $1/2$ and $-\frac1{\sqrt m}$ otherwise.
  \item \emph{Row sampling:} each row of $\S$ is
    $\frac1{\sqrt{p_im}}\,\e_i$, where $\Pr\{i\}=p_i$ and $\sum_ip_i=1$.
  \end{enumerate}
Here, the row sampling sketch can be uniform (which is common in practice), and it also includes row norm squared sampling and leverage score sampling (which leads to better results), where the distribution $p_i$ depends on the data matrix $\A$. 

Standard sketches are generally chosen so that the sketched covariance matrix $\A^\top\S^\top\S\A$ is an unbiased estimator of the full data covariance matrix, $\A^\top\A$. 
This is ensured by the fact that $\E[\S^\top\S]=\I$. 
However, in certain applications, it is not the data covariance matrix itself that is of primary interest, but rather its inverse. 
In this case, standard sketching techniques no longer yield unbiased estimators. 
Our surrogate sketches aim to correct this bias, so that, for example, we can construct an unbiased estimator for the regularized inverse covariance matrix, $(\A^\top\A+\lambda\I)^{-1}$ (given some $\lambda>0$). 
This is important for regularized least squares and second order optimization. 

We now give the definition of a surrogate sketch.  
Consider some $n$-variate measure $\mu$, and let $\X\sim\mu^m$ be the i.i.d.~random design of size $m$ for $\mu$, i.e., an $m\times n$ random matrix with i.i.d.~rows drawn from $\mu$. 
Without loss of generality, assume that $\mu$ has identity covariance, so that $\E[\X^\top\X]=m\I$. In particular, this implies that $\frac1{\sqrt m}\X\sim\Sc_\mu^m$ is a random sketching matrix.

Before we introduce the surrogate sketch, we define a so-called \emph{determinantal design} (an extension of the definitions proposed by \cite{surrogate-design,correcting-bias}), which uses determinantal rescaling to transform the distribution of $\X$ into a non-i.i.d.~random matrix $\Xb$. The transformation is parameterized by the matrix $\A$, the regularization parameter $\lambda>0$ and a parameter $\gamma>0$ which controls the size of the matrix $\Xb$.
\begin{definition}\label{d:det}
Given scalars $\lambda,\gamma> 0$ and a matrix $\A\in\R^{n\times d}$, we define the determinantal design $\Xb\sim\Det_\mu^\gamma(\A,\lambda)$ as a random matrix with randomized row-size, so that
\begin{align*}
  \Pr\big\{\Xb\in E\big\} \propto
  \E\Big[\det(\A^\top\X^\top\X\A+\lambda\gamma\I)\cdot\one_{[\X\in E]}\Big],\quad
  \text{where}\quad \X\sim \mu^{M},\quad M\sim\mathrm{Poisson}(\gamma).
\end{align*}
\end{definition}
We next give the key properties of determinantal designs that make them useful for
sketching and second-order optimization. The following lemma is an
extension of the results shown for determinantal point processes by \cite{surrogate-design}. 
\begin{lemma}\label{l:least-squares}
 Let $\Xb\sim\Det_\mu^\gamma(\A,\lambda)$. Then, we have:
  \begin{align*}
   \E\Big[\big(\A^\top\Xb^\top\Xb\A+\lambda\gamma\I\big)^{-1}\A^\top\Xb^\top\Xb\Big]
    &= \big(\A^\top\A+\lambda\I\big)^{-1}\A^\top,
    \\
   \E\Big[\big(\A^\top\Xb^\top\Xb\A+\lambda\gamma\I\big)^{-1}\Big]
    &= \gamma^{-1}\big(\A^\top\A+\lambda\I\big)^{-1}.
  \end{align*}
\end{lemma}
The row-size of $\Xb$, denoted by $\#(\Xb)$, is a random
variable, and this variable is \emph{not} distributed according to
$\mathrm{Poisson}(\gamma)$, even though $\gamma$ can be used to control its expectation. 
As a result of the determinantal rescaling, the distribution of $\#(\Xb)$ is shifted towards 
larger values relative to $\Poisson(\gamma)$, so that its expectation becomes:
\begin{align*}
  \E\big[\#(\Xb)\big] = \gamma + d_\lambda,\quad\text{where}\quad d_\lambda=\tr(\A^\top\A (\A^\top\A+\lambda\I)^{-1}).
\end{align*}
We can now define the \emph{surrogate sketching matrix} $\bar\S$ by rescaling the matrix $\Xb$, similarly to how we defined the standard sketching matrix $\S=\frac1{\sqrt m}\X$ for $\X\sim\mu^m$.
\begin{definition}\label{d:surrogate}
Let $m>d_\lambda$. Moreover, let $\gamma>0$ be
  the unique positive scalar for which $\E[\#(\Xb)]=m$, where
  $\Xb\sim\Det_\mu^{\gamma}(\A,\lambda)$. Then, $\bar\S=\frac1{\sqrt
    m}\Xb\sim\bar\Sc_\mu^m(\A,\lambda)$ is a \emph{surrogate sketching matrix}
  for $\Sc_\mu^m$.
\end{definition}
Note that many different surrogate sketches can be defined for a single sketching distribution $\Sc_\mu^m$, depending on the choice of $\A$ and $\lambda$. 
In particular, this means that a surrogate sketching distribution (even when the pre-surrogate i.i.d. distribution is Gaussian or the uniform distribution) always depends on the data matrix $\A$, whereas many standard sketches (such as Gaussian and uniform) are oblivious to the data matrix. 

Of particular interest to us is the class of surrogate \emph{row sampling} sketches, i.e. where the probability measure $\mu$ is defined by $\mu\big(\{\frac1{\sqrt{p_i}}\e_i^\top\}\big) = p_i$ for $\sum_{i=1}^n p_i=1$. In this case, we can straightforwardly leverage the algorithmic results on sampling from determinantal point processes \cite{dpp-intermediate,dpp-sublinear} to obtain efficient algorithms for constructing surrogate sketches.
\begin{theorem}\label{t:fast-surrogate}
    Given any $n\times d$ matrix $\A$, $\lambda>0$ and $(p_1,...,p_n)$, we can construct the surrogate row sampling sketch with respect to $p$ (of any size $m\leq n$) in time $O(\nnz(\A)\log(n) + d^4\log(d))$.
\end{theorem}

\section{Unbiased Estimates for the Newton Step}
\label{sec:unbiased}

Consider a convex minimization problem defined by the following loss function:
\begin{align*}
    f(\x) = \frac1n\sum_{i=1}^n\ell_i(\x^\top\varphi_i) + \frac{\lambda}{2}\|\x\|^2,
\end{align*}
where each $\ell_i$ is a twice differentiable convex function and $\varphi_1,...,\varphi_n$ are the input feature vectors in $\R^d$. 
For example, if $\ell_i(z)=\frac12(z-b_j)^2$, then we recover the regularized least squares task; and if $\ell_i(z)=\log(1+\ee^{-z b_j})$, then we recover logistic regression. 
The Newton's update for this minimization task can be written as follows:
\begin{align*}
  \x_{t+1} &= \x_t -\Big(\overbrace{\frac1n\sum_i \ell_i''(\x_t^\top\varphi_i)\,
  \varphi_i\varphi_i^\top +\lambda\I}^{\text{Hessian }\H(\x_t)}\Big)^{-1}
\Big(\overbrace{\frac1n\sum_i \ell_i'(\x^\top\varphi_i)\,\varphi_i\ +\
        \lambda\x_t}^{\text{gradient }\g(\x_t)}\Big).
\end{align*}
Newton's method can be interpreted as solving a regularized least squares problem which is the local approximation of $f$ at the current iterate $\x_t$. Thus, with the appropriate choice of matrix $\A_t$ (consisting of scaled row vectors $\varphi_i^\top$) and vector $\b_t$, the Hessian and gradient can be written as: $\H(\x_t) = \A_t^\top\A_t+\lambda\I$ and $\g(\x_t) = \A_t^\top\b_t + \lambda\x_t$. We now consider two general strategies for sketching the Newton step, both of which we discussed in Section \ref{sec:intro} for regularized least squares.

\subsection{Sketch-and-Solve}

We first analyze the classic sketch-and-solve paradigm which has been popularized in the context of least squares, but also applies directly to the Newton's method. This approach involves constructing sketched versions of both the Hessian and the gradient, by sketching with a random matrix $\S$. Crucially, we modify this classic technique by allowing the regularization parameter to be different than in the global problem, obtaining the following sketched version of the Newton step:
\begin{align*}
    \xbh_{\mathrm{SaS}} =\x_t - \Hbt_t^{-1}\tilde\g_t,\quad\text{for}\quad \Hbt_t=\A_t^\top\S^\top\S\A_t+\lambda^{\prime}\I,
    \quad \tilde\g_t = \A_t^\top\S^\top\S\b_t + \lambda^{\prime}\x_t.
\end{align*}

Our goal is to obtain an unbiased estimate of the full Newton step, i.e., such that $\E[\xbh_{\mathrm{SaS}}]=\x_{t+1}$, by combining a surrogate sketch with an appropriately scaled regularization $\lambda^{\prime}$. 

We now establish the correct choice of surrogate sketch and scaled regularization to achieve unbiasedness.
The following result is a more formal and generalized version of Theorem \ref{t:intro-unbiased}. 
We let $\mu$ be any distribution that satisfies the assumptions of Definition~\ref{d:surrogate}, so that
$\Sc_\mu^m$ corresponds to any one of the standard sketches discussed in Section \ref{s:surrogate}.
\begin{theorem}\label{t:least-squares}
  If $\xbh_{\mathrm{SaS}}$ is constructed using a surrogate sketch
  $\S\sim\bar\Sc_\mu^m(\A_t,\lambda)$ of size
  $m> d_{\lambda}$, then:
  \begin{align*}
      \E[\xbh_{\mathrm{SaS}}] = \x_{t+1}\quad\text{for}\quad\lambda^{\prime} = \lambda\cdot\Big(1 - \frac{d_{\lambda}}{m}\Big).
  \end{align*}
\end{theorem}

\subsection{Newton Sketch}
\label{sec:unbiased-ns}
We now consider the method referred to as the Newton Sketch \cite{pilanci2017newton,pilanci2016fast}, which differs from the sketch-and-solve paradigm in that it only sketches the Hessian, whereas the gradient is computed exactly. Note that in the case of least squares, this algorithm exactly reduces to the Iterative Hessian Sketch, which we discussed in Section \ref{sec:intro-theory}. This approach generally leads to more accurate estimates than sketch-and-solve, however it requires exact gradient computation, which in distributed settings often involves an additional communication round. Our Newton Sketch estimate uses the same $\lambda^{\prime}$ as for the sketch-and-solve, however it enters the Hessian somewhat differently:
\begin{align*}
    \xbh_{\mathrm{NS}} =\x_t - \Hbh_t^{-1}\g(\x_t),\quad\text{for}\quad \Hbh_t=\tfrac\lambda{\lambda^{\prime}}\A_t^\top\S^\top\S\A_t+\lambda\I = \tfrac{\lambda}{\lambda^{\prime}}\Hbt_t.
\end{align*}
The additional factor $\frac{\lambda}{\lambda^{\prime}}$ comes as a result of using the exact gradient. One way to interpret it is that we are scaling the data matrix $\A_t$ instead of the regularization. The following result shows that, with $\lambda^{\prime}$ chosen as before, the surrogate Newton Sketch is unbiased.
\begin{theorem}\label{t:hessian}
  If $\xbh_{\mathrm{NS}}$ is constructed using a surrogate sketch
  $\S\sim\bar\Sc_\mu^m(\A_t,\lambda)$ of size
  $m> d_{\lambda}$, then:
  \begin{align*}
      \E[\xbh_{\mathrm{NS}}] = \x_{t+1}\quad\text{for}\quad\lambda^{\prime} = \lambda\cdot\Big(1 - \frac{d_{\lambda}}{m}\Big).
  \end{align*}
\end{theorem}

\section{Convergence Analysis}
\label{sec:convergence_analysis}

Here, we study the convergence guarantees of the surrogate Newton Sketch with distributed averaging. Consider $q$ i.i.d.~copies $\Hbh_{t,1},...,\Hbh_{t,q}$ of the Hessian sketch $\Hbh_t$ defined in Section \ref{sec:unbiased-ns}. We start by finding an upper bound for the distance between the optimal Newton update and averaged Newton sketch update at the $t$'th iteration, defined as $\xbh_{t+1}=\x_t-\frac1q\sum_{k=1}^q\Hbh_{t,k}^{-1}\g(\x_t)$. We will use Mahalanobis norm as the distance metric. Let $\|\v\|_{\M}$ denote the Mahalanobis norm, i.e., $\|\v\|_{\M} = \sqrt{\v^\top \M \v}$. The distance between the updates is equal to the distance between the next iterates:
\begin{align*}
    \|\x_{t+1} - \xbh_{t+1} \|_{\H_t} = \| (\bar\H_t^{-1} - \H_t^{-1}) \g_t \|_{\H_t},\quad\text{where}\quad\bar\H_t^{-1} = \frac1q\sum_{k=1}^q\Hbh_{t,k}^{-1}.
\end{align*}
We can bound this quantity in terms of the spectral norm approximation error of $\bar\H_t^{-1}$ as follows:
\begin{align*}
    \|(\bar\H_t^{-1} - \H_t^{-1}) \g_t \|_{\H_t} 
    \leq \| \H_t^{\frac{1}{2}} (\bar\H_t^{-1} - \H_t^{-1}) \H_t^{\frac{1}{2}} \| \cdot \| \H_t^{-1} \g_t \|_{\H_t}.
\end{align*}
Note that the second term, $\H_t^{-1} \g_t$, is the exact Newton step. 
To upper bound the first term, we now focus our discussion on a particular variant of surrogate sketch $\bar\Sc_\mu^m(\A,\lambda)$ that we call surrogate leverage score sampling. Leverage score sampling is an i.i.d.~row sampling method, i.e., the probability measure $\mu$ is defined by $\mu\big(\{\frac1{\sqrt{p_i}}\e_i^\top\}\big) = p_i$ for $\sum_{i=1}^n p_i=1$. Specifically, we consider the so-called $\lambda$-ridge leverage scores which have been used in the context of regularized least squares \cite{ridge-leverage-scores}, where the probabilities must satisfy $p_i\geq\frac12\a_i^\top(\A_t^\top\A_t+\lambda\I)^{-1}\a_i/d_\lambda$ ($\a_i^\top$ denotes a row of $\A_t$). Such $p_i$'s can be found efficiently using standard random projection techniques \cite{fast-leverage-scores,cw-sparse}.
\begin{lemma} \label{l:matrix_concent}
    If $n\geq m\geq C\alpha d_\lambda\polylog(n,\kappa,1/\delta)$ and we use the surrogate leverage score sampling sketch of size $m$, then the i.i.d.~copies $\Hbh_{t,1},...,\Hbh_{t,q}$ of the sketch $\Hbh_t$  with probability $1-\delta$ satisfy:
    \begin{align*}
        \| \H_t^{\frac{1}{2}} (\bar\H_t^{-1} - \E[\bar\H_t^{-1}]) \H_t^{\frac{1}{2}} \|
        \leq \frac{1}{\sqrt{\alpha q}},\qquad\text{where}\quad \bar\H_t^{-1} = \frac1q\sum_{k=1}^q\Hbh_{t,k}^{-1}.
    \end{align*}
\end{lemma}
Note that, crucially, we can invoke the unbiasedness of the Hessian sketch, $\E[\bar\H_t^{-1}]=\H_t^{-1}$, so we obtain that with probability at least $1-\delta$,
\begin{align} \label{eq:dist_next_iter}
    \|\x_{t+1} - \xbh_{t+1} \|_{\H_t} \leq  \frac{1}{\sqrt{\alpha q}} \cdot\| \H_t^{-1} \g_t \|_{\H_t}.
\end{align}
We now move on to measuring how close the next Newton sketch iterate is to the global optimizer of the loss function $f(\x)$. For this part of the analysis, we assume that the Hessian matrix is $L$-Lipschitz.
\begin{assumption} \label{a:hessian_lipschitz}
    The Hessian matrix $\H(\x)$ is $L$-Lipschitz continuous, that is, $\| \H(\x) - \H(\x^\prime) \| \leq L \| \x - \x^\prime \|$ for all $\x$ and $\x^\prime$.
\end{assumption}
Combining \eqref{eq:dist_next_iter} with Lemma 14 from \cite{determinantal-averaging} and letting $\eta=1/\sqrt\alpha$, we obtain the following convergence result for the distributed Newton Sketch using surrogate leverage score sampling sketch.
\begin{theorem}\label{t:newton}
Let $\kappa$ and $\lambda_{\min}$ be the condition number and smallest eigenvalue of the Hessian $\H(\x_t)$, respectively. The distributed Newton Sketch update constructed using a surrogate leverage score sampling sketch of size $m=O(\alpha d\cdot \polylog(n,\kappa,1/\delta))$ and averaged over $q$ workers, satisfies:
\begin{align*}
    \| \xbh_{t+1} - \x^* \| \leq \max \left\{ \frac{1}{\sqrt{\alpha q}} \sqrt{\kappa} \|\x_t - \x^*\|, \frac{2L}{\lambda_{\min}} \| \x_t - \x^* \|^2 \right\}.
\end{align*}
\end{theorem}
\begin{remark}
The convergence rate for the distributed Iterative Hessian Sketch algorithm as given in Theorem \ref{t:main} is obtained by using \eqref{eq:dist_next_iter} with $\H_t=\A^\top\A + \lambda \I$. The assumption that $\alpha\geq d$ in Theorem~\ref{t:main} is only needed for the time complexity (see Theorem \ref{t:fast-surrogate}). The convergence rate holds for $\alpha\geq 1$.
\end{remark} 
\section{Numerical Results}
\label{sec:numerical}
In this section we present numerical results, with further details provided in Appendix \ref{sec:additional_numerical}. Figures \ref{fig:regularized_ls_cifar} and \ref{fig:regularized_ls_detavg_comparison} show the estimation error as a function of the number of averaged outputs for the regularized least squares problem discussed in Section \ref{sec:intro-debiasing}, on Cifar-10 and Boston housing prices datasets, respectively. 

\begin{figure} 
\begin{minipage}[b]{0.32\linewidth}
  \centering
  \centerline{\includegraphics[width=\columnwidth]{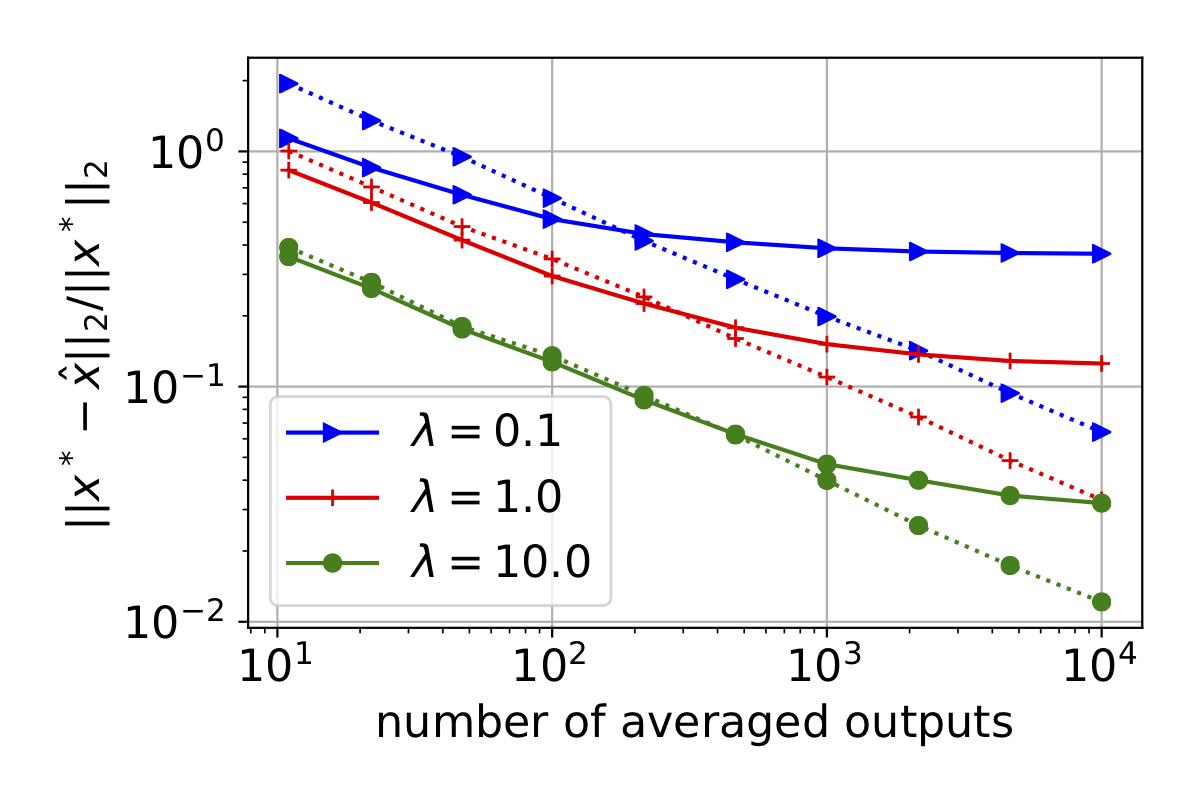}}
  \centerline{(a) Gaussian}\medskip
\end{minipage}
\hfill
\begin{minipage}[b]{0.32\linewidth}
  \centering
  \centerline{\includegraphics[width=\columnwidth]{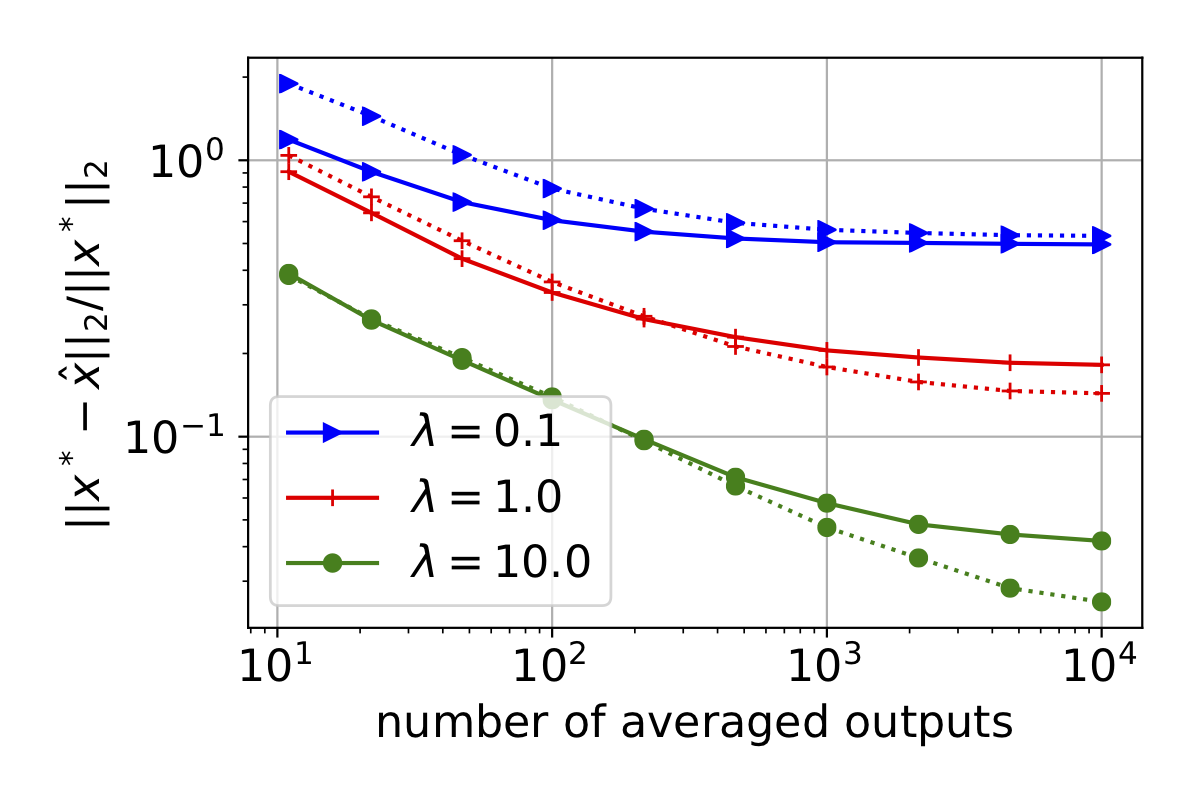}}
  \centerline{(b) Uniform}\medskip
\end{minipage}
\hfill
\begin{minipage}[b]{0.32\linewidth}
  \centering
  \centerline{\includegraphics[width=\columnwidth]{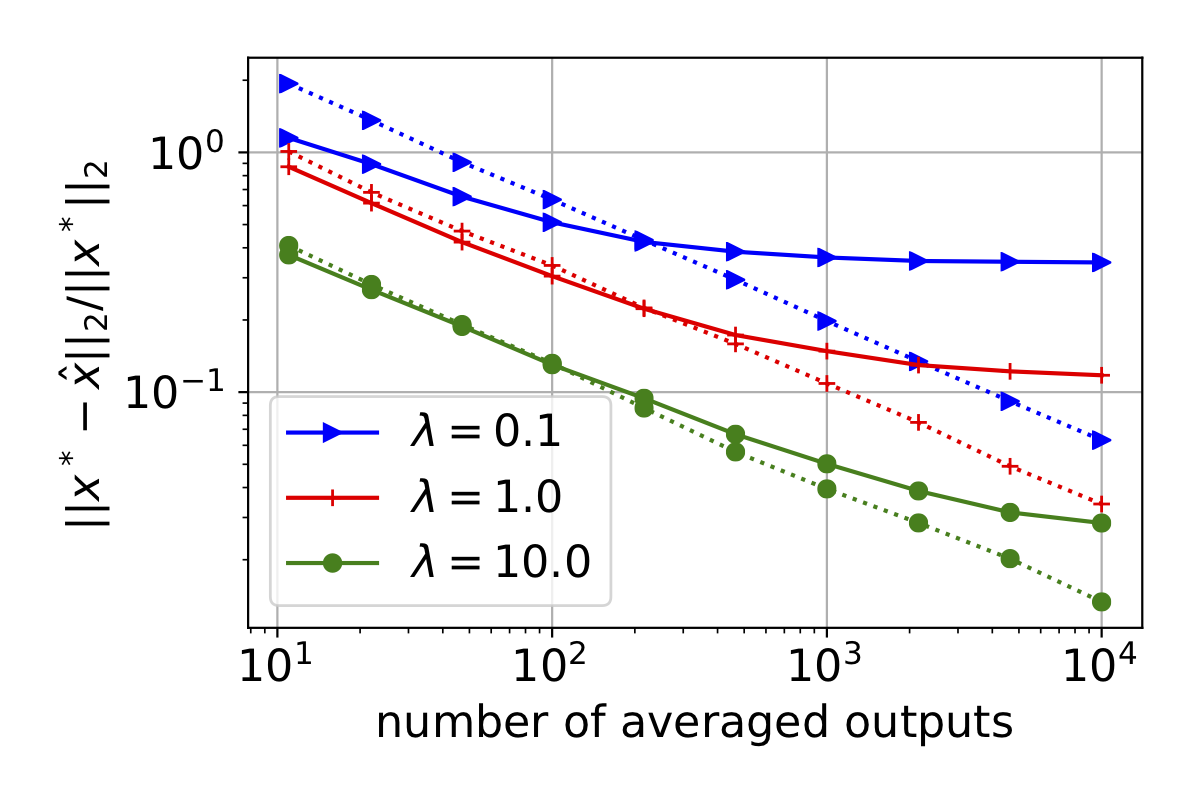}}
  \centerline{(c) Surrogate sketch}\medskip
\end{minipage}
\vspace{-2mm}
\caption{Estimation error against the number of averaged outputs for regularized least squares on first two classes of Cifar-10 dataset ($n=10000$, $d=3072$, $m=1000$) for different regularization parameter values $\lambda$. The dotted lines show the error for the debiased versions (obtained using $\lambda^{\prime}$ expressions) for each straight line with the same color and marker.}
\vspace{-3mm}
\label{fig:regularized_ls_cifar}
\end{figure}

\begin{figure} 
\begin{minipage}[b]{0.32\linewidth}
  \centering
  \centerline{\includegraphics[width=\columnwidth]{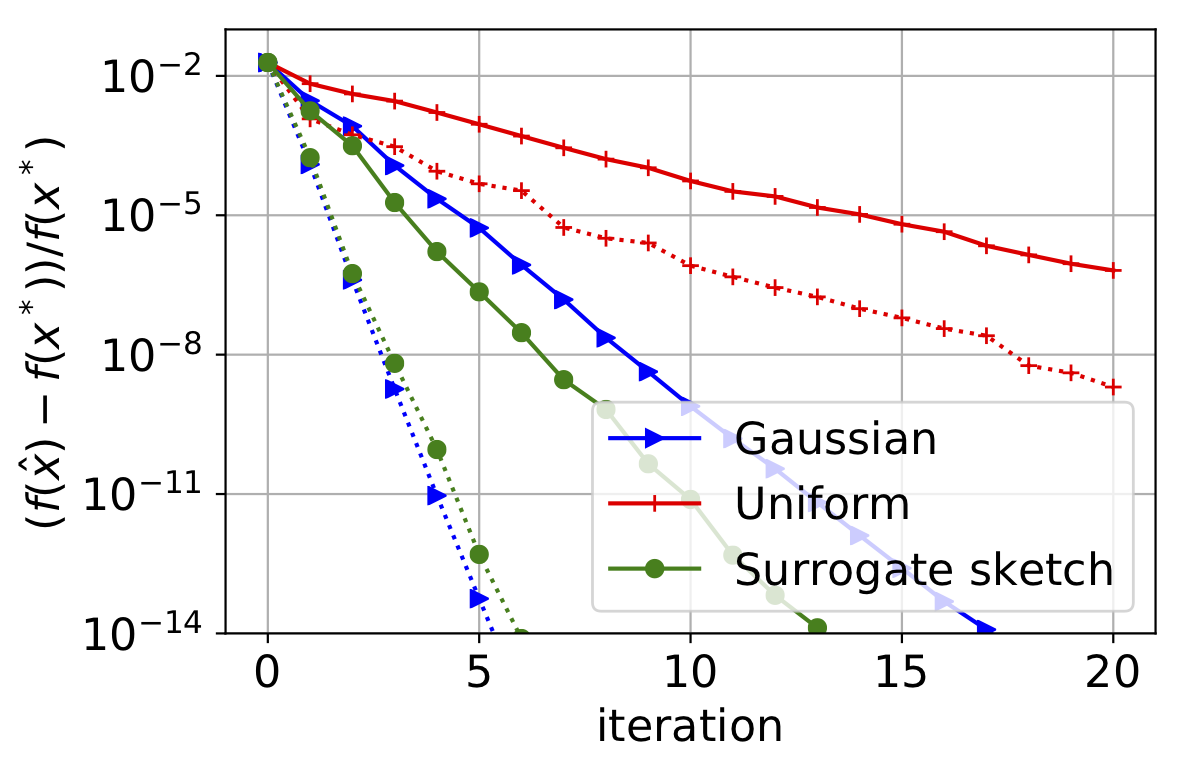}}
  \centerline{(a) statlog-australian-credit }\medskip
\end{minipage}
\hfill
\begin{minipage}[b]{0.32\linewidth}
  \centering
  \centerline{\includegraphics[width=\columnwidth]{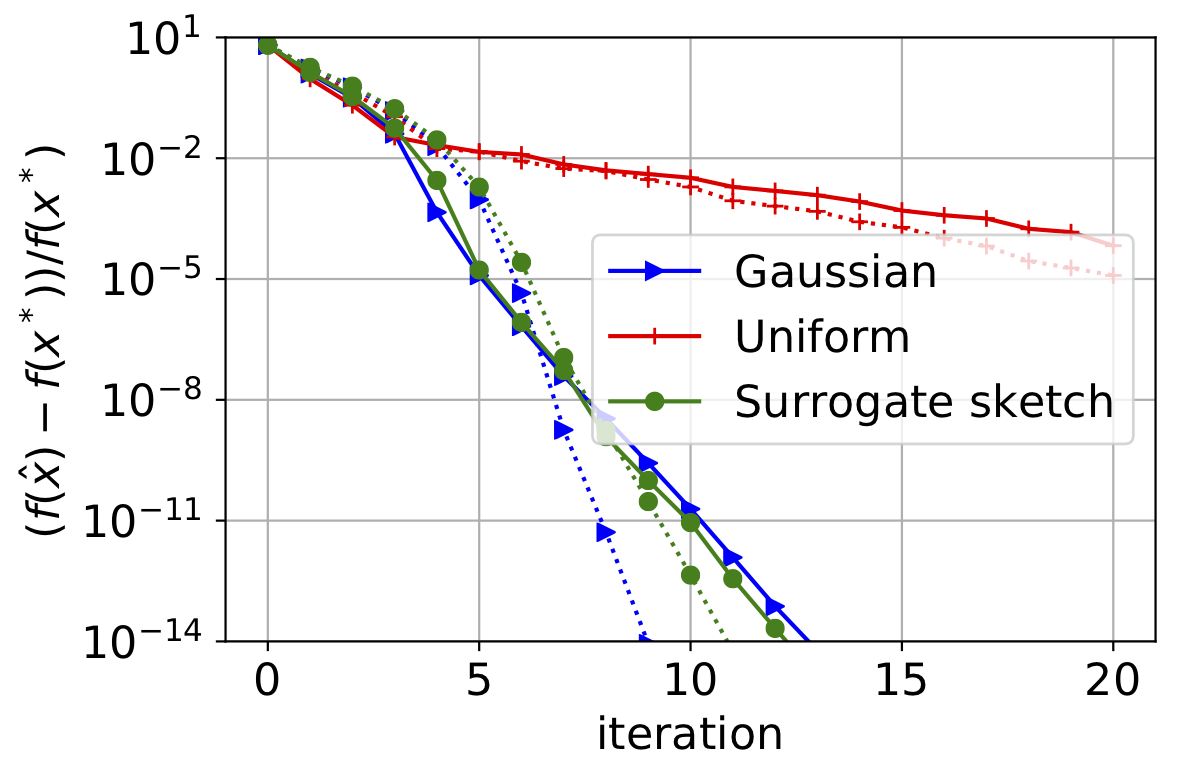}}
  \centerline{(b) breast-cancer-wisc }\medskip
\end{minipage}
\hfill
\begin{minipage}[b]{0.32\linewidth}
  \centering
  \centerline{\includegraphics[width=\columnwidth]{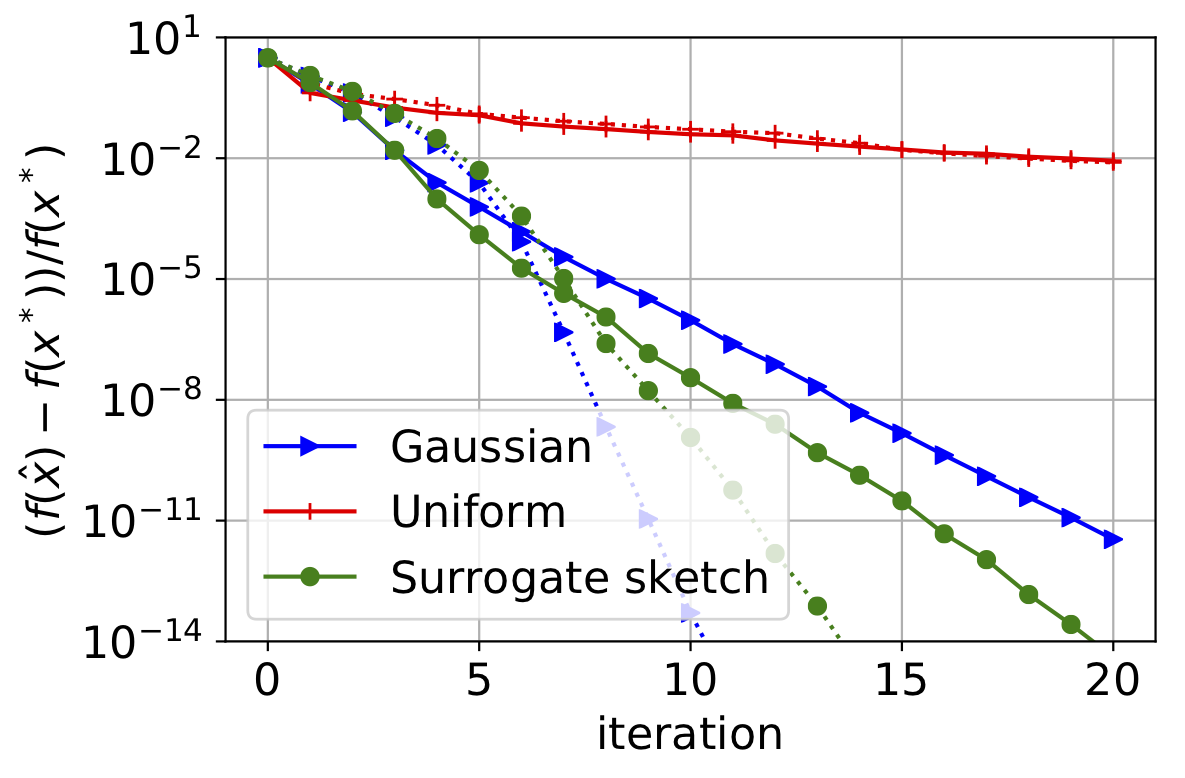}}
  \centerline{(c) ionosphere }\medskip
\end{minipage}
\vspace{-2mm}
\caption{Distributed Newton Sketch algorithm for logistic regression with $\ell_2$-regularization on different UCI datasets. The dotted curves show the error for when the regularization parameter is rescaled using the provided expression for $\lambda^{\prime}$. In all the experiments, we have $q=100$ workers and $\lambda=10^{-4}$. The dimensions for each dataset are ($690\times 14$), ($699\times 9$), ($351\times 33$), and the sketch sizes are $m=50,50,100$ for plots a,b,c, respectively. The step size for distributed Newton sketch updates has been determined via backtracking line search with parameters $\tau=2$, $c=0.1$, $a_0=1$. }
\vspace{-3mm}
\label{fig:dist_newton_sketch}
\end{figure}

\ifthenelse{\boolean{arxivVersion}}{
\begin{figure}
\centering
\centerline{\includegraphics[width=0.4\textwidth]{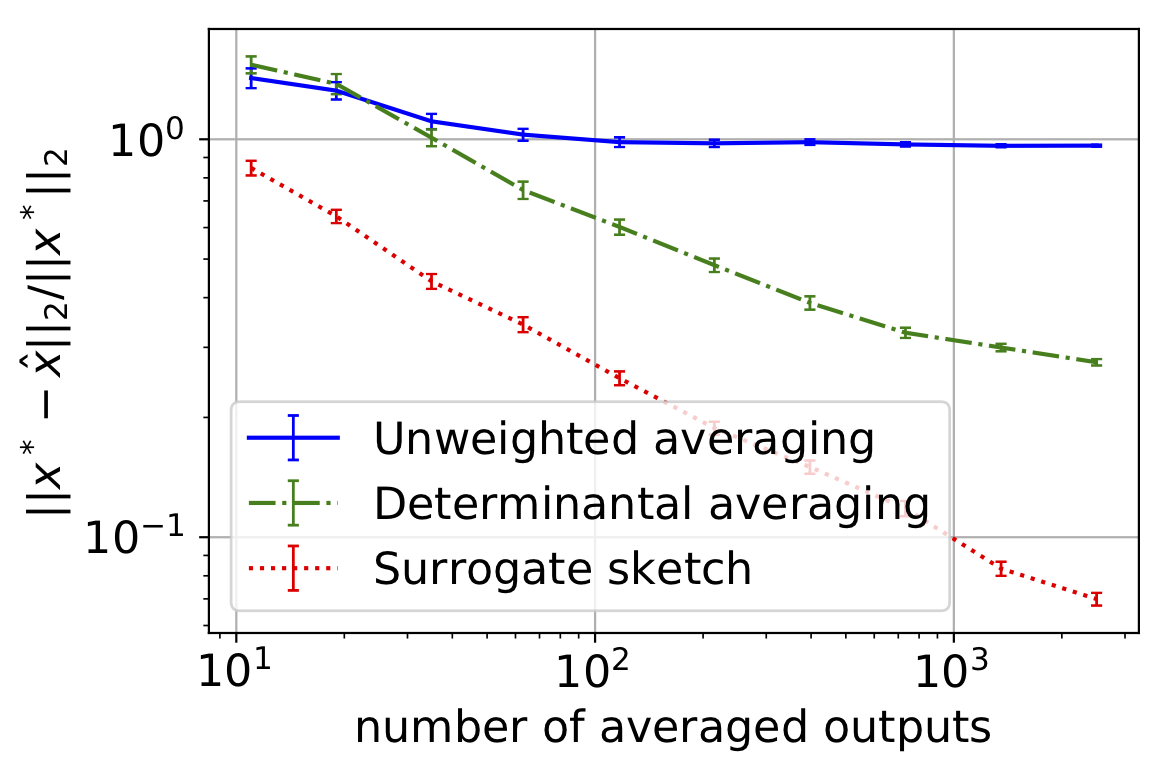}}
\caption{Estimation error of the \emph{surrogate sketch}, against uniform sampling with \emph{unweighted averaging} \cite{mahoney2018giant} and \emph{determinantal averaging} \cite{determinantal-averaging}.}
\label{fig:regularized_ls_detavg_comparison}
\end{figure}
}{
\begin{wrapfigure}{r}{0.4\textwidth}
\centering
\vspace{-1mm}
\centerline{\includegraphics[width=0.38\textwidth]{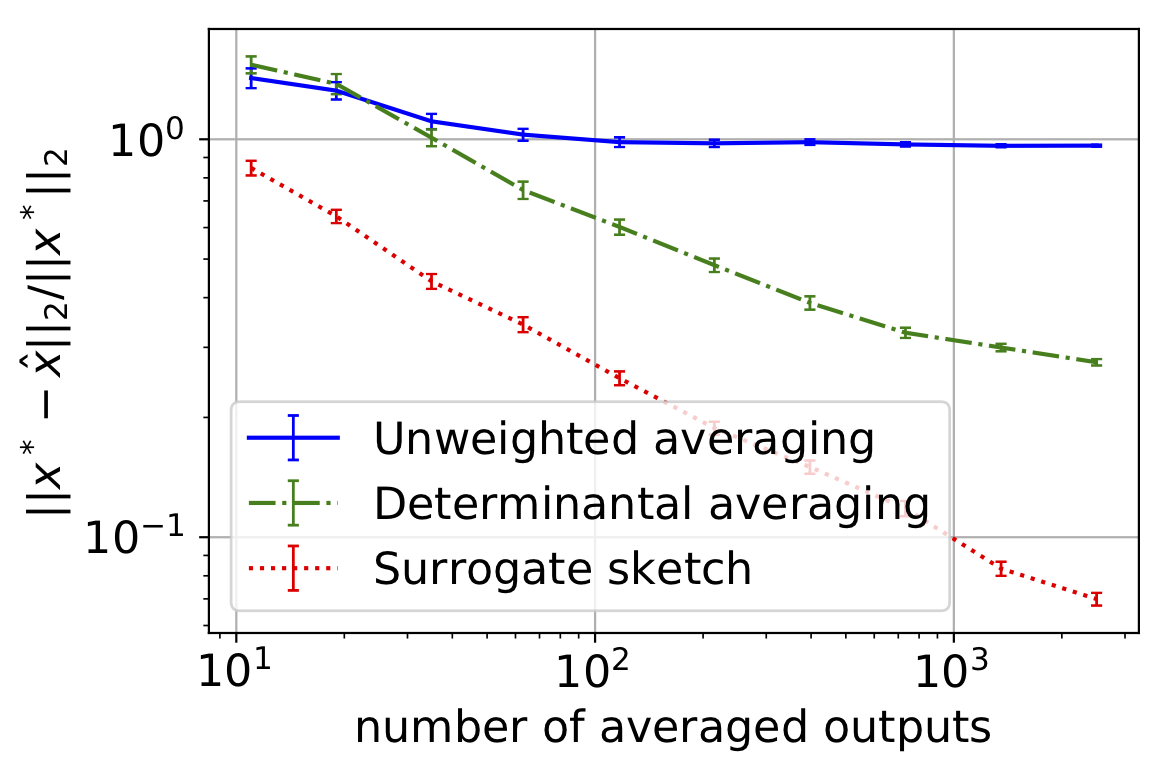}}
\vspace{-1mm}
\caption{Estimation error of the \emph{surrogate sketch}, against uniform sampling with \emph{unweighted averaging} \cite{mahoney2018giant} and \emph{determinantal averaging} \cite{determinantal-averaging}.}
\vspace{-3mm}
\label{fig:regularized_ls_detavg_comparison}
\end{wrapfigure}
}

Figure \ref{fig:regularized_ls_cifar} illustrates that  when the number of averaged outputs is large, rescaling the regularization parameter using the expression $\lambda^{\prime}=\lambda\cdot(1-\frac{d_\lambda}{m})$, as in Theorem \ref{t:intro-unbiased}, improves on the estimation error for a range of different $\lambda$ values. We observe that this is true not only for the surrogate sketch but also for the Gaussian sketch (we also tested the Rademacher sketch, which performed exactly as the Gaussian did). For uniform sampling, rescaling the regularization parameter does not lead to an unbiased estimator, but it significantly reduces the bias in most instances. 
Figure \ref{fig:regularized_ls_detavg_comparison} compares the surrogate row sampling sketch  to the standard i.i.d.~row sampling used in conjunction with averaging methods suggested by \cite{mahoney2018giant} (unweighted averaging) and \cite{determinantal-averaging} (determinantal averaging), on the Boston housing dataset. 
We used: $\lambda = 10$, $\lambda^{\prime}=4.06$, and sketch size $m=50$. We show an average over 100 trials, along with the standard error. We observe that the better theoretical guarantees achieved by the surrogate sketch, as shown in Table~\ref{tab:convergence}, translate to improved empirical performance.

Figure \ref{fig:dist_newton_sketch} shows the estimation error against iterations for the distributed Newton sketch algorithm running on a logistic regression problem with $\ell_2$ regularization on three different binary classification UCI datasets. We observe that the rescaled regularization technique leads to significant speedups in convergence, particularly for Gaussian and surrogate sketches.

\section{Conclusion}
\label{sec:conclusion}
We introduced two techniques for debiasing distributed second order methods. First, we defined a family of sketching methods called \emph{surrogate sketches}, which admit exact bias expressions for local Newton estimates. Second, we proposed \emph{scaled regularization}, a method for correcting that bias.  

\section*{Acknowledgements}
This work was partially supported by the National Science Foundation under grant IIS-1838179.
Also, MD and MWM acknowledge DARPA, NSF, and ONR for providing partial support of this work.

\bibliographystyle{plain}
\bibliography{pap,REFS_1}

\newpage
\appendix
\section{Expectation Formulas for Surrogate Sketches}
\label{sec:expectation_formulas}

In this section show the expectation formulas given in Lemma \ref{l:least-squares}. First, we derive the normalization constant of the determinantal design introduced in Definition \ref{d:det}. For this, we rely on the framework of \emph{determinant preserving random matrices} recently introduced by \cite{surrogate-design}. The proofs here roughly follow the techniques from \cite{surrogate-design}, the main difference being that we consider regularized matrices, whereas they focus on the unregularized case. 

A square random matrix is determinant preserving (d.p.) if taking expectation commutes with computing a determinant for that matrix and all its submatrices. Consider the matrix $\X\sim\mu^M$ for an isotropic $d$-variate measure $\mu$ and $M\sim\mathrm{Poisson}(\gamma)$, as in Definition \ref{d:det}. In Lemma 5, \cite{surrogate-design} show that the matrix $\A^\top\X^\top\X\A$ is determinant preserving. Thus, using closure under addition (Lemma 4 in \cite{surrogate-design}), the matrix $\A^\top\X^\top\X\A + \lambda\gamma\I$ is also d.p., so the normalization constant for the probability defined in Definition \ref{d:det} is:
\begin{align*}
    \E\big[\det(\A^\top\X^\top\X\A+\lambda\gamma\I)\big] 
    = \det\big(\A^\top\E[\X^\top\X]\A+\lambda\gamma\I\big) = \det(\gamma\A^\top\A+\lambda\gamma\I).
\end{align*}
\begin{proofof}{Lemma}{\ref{l:least-squares}}
By definition, any d.p.~matrix $\C$ satisfies $\E[\adj(\C)]=\adj(\E[\C])$, where $\adj(\C)$ denotes the adjugate of a square matrix, which for any positive definite matrix is given by $\adj(\C)=\det(\C)\C^{-1}$. This allows us to show the second expectation formula from Lemma \ref{l:least-squares} for $\Xb\sim\Det_\mu^\gamma(\A,\lambda)$. Note that the proof is analogous to the proof of Lemma 11 in \cite{surrogate-design}.
\begin{align*}
    \E\Big[\big(\A^\top\Xb^\top\Xb\A + \lambda\gamma\I\big)^{-1}\Big]
    &=\frac{\E[\det(\A^\top\X^\top\X\A+\lambda\gamma\I)\cdot(\A^\top\X^\top\X\A+\lambda\gamma\I)^{-1}]}
    {\det(\gamma\A^\top\A+\lambda\gamma\I)}
    \\
    &=\frac{\E[\adj(\A^\top\X^\top\X\A+\lambda\gamma\I)]}{\det(\gamma\A^\top\A+\lambda\gamma\I)}
    \\
    &=\frac{\adj(\gamma\A^\top\A+\lambda\gamma\I)}{\det(\gamma\A^\top\A+\lambda\gamma\I)}
    =(\gamma\A^\top\A+\lambda\gamma\I)^{-1}.
\end{align*}
We next prove the first expectation formula from Lemma \ref{l:least-squares}, by following the steps outlined by \cite{surrogate-design} in the proof of their Lemma 13. Let $\b$ denote any vector in $\R^n$. The $i$th entry of the vector $\big(\A^\top\X^\top\X\A+\lambda\gamma\I\big)^{-1}\A^\top\X^\top\X\b$ can be obtained by left multiplying it by the $i$th standard basis vector $\e_i^\top$. We will use the following observation (Fact 2.14.2 from \cite{matrix-mathematics}):
\begin{align*}
    \e_i^\top\adj\big(\A^\top\X^\top\X\A+\lambda\gamma\I\big)^{-1}\A^\top\X^\top\X\b
    &= \det(\A^\top\X^\top\X\A+\lambda\gamma\I + \A^\top\X^\top\X\b\e_i^\top)
    \\
    & \quad -\det(\A^\top\X^\top\X\A+\lambda\gamma\I).
\end{align*}
Combining this with the fact that both the matrices $\A^\top\X^\top\X\A+\lambda\gamma\I$ and $\A^\top\X^\top\X(\A+\b\e_i^\top)+\lambda\gamma\I$ are determinant preserving for $\X$ defined as before, we obtain that:
\begin{align*}
    \E\big[&\e_i^\top\big(\A^\top\Xb^\top\Xb\A+\lambda\gamma\I\big)^{-1}\A^\top\Xb^\top\Xb\b\big]
    =\frac{\E\big[\e_i^\top\adj\big(\A^\top\X^\top\X\A+\lambda\gamma\I\big)^{-1}\A^\top\X^\top\X\b\big]}
    {\det(\gamma\A^\top\A+\lambda\gamma\I)}
    \\
    &=\frac{\E[\det(\A^\top\X^\top\X(\A+\b\e_i^\top)+\lambda\gamma\I\big)
    -\det(\A^\top\X^\top\X\A+\lambda\gamma\I\big)]}{\det(\gamma\A^\top\A+\lambda\gamma\I)}
    \\
    &=\frac{\det(\gamma\A^\top(\A+\b\e_i^\top)+\lambda\gamma\I\big)-\det(\gamma\A^\top\A+\lambda\gamma\I)}
    {\det(\gamma\A^\top\A+\lambda\gamma\I)}
    \\
    &=\frac{\det(\gamma\A^\top\A+\lambda\gamma\I)\e_i^\top(\gamma\A^\top\A+\lambda\gamma\I)^{-1}\gamma\A^\top\b}
    {\det(\gamma\A^\top\A+\lambda\gamma\I)}
    = \e_i^\top(\A^\top\A+\lambda\I)^{-1}\A^\top\b.
\end{align*}
Since the above holds for all indices $i$ and all vectors $\b$, this completes the proof.
\end{proofof}

We next use Lemma \ref{l:least-squares} to prove Theorems \ref{t:least-squares} and \ref{t:hessian}.
\begin{proofof}{Theorem}{\ref{t:least-squares}}
Suppose that, in Lemma \ref{l:least-squares}, parameter $\gamma$ is chosen as in Definition
\ref{d:surrogate} and let $\Xb\sim\Det_\mu^\gamma(\A,\lambda)$. Then we have
$m=\E[\#(\Xb)] = \gamma + d_{\lambda}(\A)$ and the surrogate sketch
in Theorem \ref{t:least-squares} is given by $\S=\frac1{\sqrt m}\Xb$. 
Note that we have $\lambda^\prime = \lambda\cdot\gamma/m$, so we can write:
\begin{align*}
  \E\big[\Hbt_t^{-1}\tilde\g_t\big]
  &= \E\Big[\big(\A_t^\top\S^\top\S\A_t+\lambda^{\prime}\I\big)^{-1}
\big(\A_t^\top\S^\top\S\b_t + \lambda^\prime\x_t\big)\Big]\\
  &= \E\Big[\Big(\tfrac1m\A_t^\top\Xb^\top\Xb\A_t
    +\lambda\tfrac{\gamma}{m}\,\I\Big)^{-1}
    \!\tfrac1m\A_t^\top\Xb^\top\Xb\b_t\Big] + 
    \E\Big[\Big(\tfrac1m\A_t^\top\Xb^\top\Xb\A_t
    +\lambda\tfrac{\gamma}{m}\,\I\Big)^{-1}
    \!\lambda\tfrac\gamma m \x_t\Big]\\
  &= \E\Big[\Big(\A_t^\top\Xb^\top\Xb\A_t
    +\lambda\gamma\,\I\Big)^{-1}
    \!\A_t^\top\Xb^\top\Xb\Big]\b_t + 
    \E\Big[\Big(\A_t^\top\Xb^\top\Xb\A_t
    +\lambda\gamma\,\I\Big)^{-1}\Big]
    \lambda\gamma\x_t\\
  &\overset{(*)}{=}(\A_t^\top\A_t+\lambda\I)^{-1}\A_t^\top\b 
  + (\A_t^\top\A_t+\lambda\I)^{-1}\lambda\x_t 
  = \H^{-1}(\x_t)\g(\x_t),
\end{align*}
where in $(*)$ we used both formulas from Lemma \ref{l:least-squares}. 
This concludes the proof.
\end{proofof}

The proof of Theorem \ref{t:hessian} follows analogously.
\begin{proofof}{Theorem}{\ref{t:hessian}}
Suppose that, in Lemma \ref{l:least-squares}, parameter $\gamma$ is chosen as in Definition
\ref{d:surrogate} and let $\Xb\sim\Det_\mu^\gamma(\A,\lambda)$,
with $\S=\frac1{\sqrt m}\Xb$.
Once again, we have $\lambda^\prime = \lambda\cdot\gamma/m$, so we can write:
\begin{align*}
  \E\big[\Hbh_t^{-1}\big]
  &= \E\Big[\big(\tfrac m\gamma\A_t^\top\S^\top\S\A_t+\lambda\I\big)^{-1}\Big]
  = \gamma\E\Big[\big(\A_t^\top\Xb^\top\Xb\A_t
    +\lambda\gamma\,\I\big)^{-1}\Big]
  \overset{(*)}{=}(\A_t^\top\A_t+\lambda\I)^{-1},
\end{align*}
where in $(*)$ we used the second formula from Lemma \ref{l:least-squares}. 
This concludes the proof.
\end{proofof}

\section{Efficient Algorithms for Surrogate Sketches}
\label{app:fast-dpp}
In this section, we provide a framework for implementing surrogate sketches by relying on the algorithmic techniques from the DPP sampling literature. We then use these results to give the input-sparsity time implementation of the surrogate leverage score sampling sketch.
\begin{definition}
  Given a probability measure $\mu$ over domain $\Omega$ and a kernel
  function $K:\Omega\times\Omega\rightarrow\R_{\geq 0}$, we define a
  determinantal point process $\Xc\sim\DPP_{\mu}(K)$ as a distribution over
  finite subsets of $\Omega$, such that for any $k\geq 0$ and event $E\subseteq
  {\Omega \choose k}$:
  \begin{align*}
    \Pr\{\Xc\in E\}\propto
\frac1{k!}\,\E_{\mu^k}\Big[\one_{[\{x_1,...,x_k\}\in
    E]}\det\!\big([K(x_i,x_j)]_{ij}\big)\Big].
  \end{align*}
\end{definition}
\begin{remark}
If $\Omega$ is the set of row vectors of an $n\times d$ matrix $\A$ and
$K(\a_1,\a_2)=\a_1^\top\a_2$, then $\Xc\sim\DPP_{\mu}(K)$, with $\mu$ being the
uniform measure over $\Omega$, reduces to a standard
L-ensemble DPP \cite{dpp-ml}. In particular, let $S\sim\DPP(\A\A^\top)$ denote a random subset of $[n]$ sampled so that $\Pr(S)\propto\det(\A_S\A_S^\top)$. Then the set of rows of $\A_S$ is distributed identically to $\Xc$.
\end{remark}
A key property of L-ensembles, which relates them to the $\lambda$-effective dimension is that if $\Xc\sim\DPP_{\mu}(K)$ for an isotropic measure $\mu$ and $K(\x,\y) = \frac1\lambda\x^\top\A\A^\top\y$, then $\E[|\Xc|]=d_\lambda(\A)$.

We next show that our determinantal design (Definition \ref{d:det}) can be decomposed into a DPP portion and an i.i.d.~portion, which enables efficient sampling for surrogate sketches. A similar result was previously shown by \cite{surrogate-design} for their determinantal design (which is different from ours in that it is not regularized). The below result also immediately leads to the formula for the expected size of a surrogate sketch $\Xb\sim\Det_\mu^\gamma(\A,\lambda)$ given in Section~\ref{s:surrogate}, which states that $\E[\#(\Xb)]=\gamma+d_\lambda(\A)$.
\begin{lemma}\label{l:composition}
Let $\mu$ be a probability measure over $\R^n$. Given scalars
$\lambda,\gamma>0$ and a matrix $\A\in\R^{n\times d}$, let
$\Xc\sim\DPP_{\mu}(K)$, where $K(\x,\y) =
\frac1\lambda\x^\top\A\A^\top\y$ and $\X\sim\mu^M$ for
$M\sim\Poisson(\gamma)$. Then the matrix $\Xb$ formed by adding the elements
of $\Xc$ as rows into the matrix $\X$ and then randomly permuting the
rows of the obtained matrix, is distributed as $\Det_\mu^\gamma(\A,\lambda)$.
\end{lemma}
\begin{proof}
  Let $E\subseteq \R^{t\times n}$ be an event measurable with respect
  to $\mu^t$. We have:
  \begin{align*}
\Pr\{\Xb\in E\}
 &=
 \sum_{s=0}^t\frac{\gamma^{t-s}\ee^{-\gamma}}{(t-s)!}\cdot\frac1{{t\choose s}}
\sum_{S\in{[t]\choose s}}
\frac{\E_{\mu^t}[\one_{[\X\in E]}\cdot\det(\frac1\lambda\X_S\A\A^\top\X_S^\top)]}
{s!\,\det(\frac1\lambda\A^\top\A+\I)}\\
    &=\frac{\gamma^t\ee^{-\gamma}}{t!}\sum_{s=0}^t\sum_{S\in{[t]\choose s}}
      \frac{\E_{\mu^t}[\one_{[\X\in E]}\cdot\det(\frac1{\lambda\gamma}\X_S\A\A^\top\X_S^\top)]}
{\det(\frac1\lambda\A^\top\A+\I)}\\
    &=\frac{\gamma^t\ee^{-\gamma}}{t!
      \det(\frac1\lambda\A^\top\A+\I)}\,
      \E_{\mu^t}\bigg[\one_{[\X\in E]}\cdot\sum_{S\subseteq
      [t]}\det\big(\tfrac1{\lambda\gamma}\X_S\A\A^\top\X_S^\top\big)\bigg]\\
&=\frac{\gamma^t\ee^{-\gamma}}{t!
\det(\frac1\lambda\A^\top\A+\I)}\E_{\mu^t}
\Big[\one_{[\X\in E]}\cdot\det\big(
\tfrac1{\lambda\gamma}\A^\top\X^\top\X\A + \I\big)\Big]\\
&= \frac{\gamma^t\ee^{-\gamma}}{t!}\cdot\frac{\E_{\mu^t}\big[\one_{[\X\in E]}\cdot \det(\A^\top\X^\top\X\A+\lambda\gamma\I)\big]}{\det(\gamma\A^\top\A+\lambda\gamma\I)},
  \end{align*}
  which concludes the proof.
\end{proof}
We now give an algorithm for sampling a surrogate of the
i.i.d.~row-sampling sketch where the importance sampling distribution
is given by $p:[n]\rightarrow\R_{\geq 0}$. Here, the probability
measure $\mu$ is defined so that
$\mu\big(\{\frac1{\sqrt{p_i}}\e_i^\top\}\big) = p_i$
for each $i\in[n]$. The surrogate sketch
$\bar\Sc_\mu^m(\A,\lambda)$ for this $\mu$
can be constructed as follows:
\begin{enumerate}
  \item Sample set $S\subseteq[n]$ so that $\Pr\{S\}\propto
    \det(\frac1\lambda\A_{S}\A_{S}^\top)$, i.e., according to $\DPP(\frac1\lambda\A\A^\top)$.
  \item Draw $M\sim\Poisson(\gamma)$ for $\gamma = m-d_\lambda$ and sample $i_1,...,i_M$
    i.i.d.~from $p$.
\item Let $\sigma_1,...,\sigma_{M+|S|}$ be a sequence consisting of $S$ and
  $i_1,...,i_M$ randomly permuted together. 
\item Then, the $i$th row of $\bar
  \S$ is $\frac1{\sqrt {m p_{\sigma_i}}}\e_{\sigma_i}^\top$ for $i\in[M+|S|]$.
\end{enumerate}

We next present an implementation of the surrogate row sampling sketch which runs in input-sparsity time for tall matrices $\A$ (i.e., when $n\gg d$). 
The algorithm samples exactly from the surrogate sketching distribution, which is crucial for the analysis. Our algorithm is based on two recent papers on DPP sampling \cite{dpp-intermediate,dpp-sublinear}, however we use a slight modification due to \cite{alpha-dpp}, which ensures exact sampling in input sparsity time.
\begin{algorithm}
\caption{Sampling from $\DPP(\A\A^\top)$}\label{alg:dpp}
  \begin{algorithmic}[1]
    \STATE \textbf{input:} $\A\in\R^{n\times d}$, \
    $\C\in\R^{d\times d}$, $(\tilde{l}_1,\dots,\tilde{l}_n)$,
    $s =\sum_i\tilde l_i$, $\tilde s=\tr\big(\C(\C+\I)^{-1}\big)$\\[1mm] 
    \STATE \textbf{repeat}\label{line:rep1}
    \vspace{1mm}
    \STATE \quad sample $u \sim
    \mathrm{Poisson}(r\ee^{1/r}2s)$\label{line:poisson1}
    \vspace{1mm}
    \STATE \quad sample $\rho_1,\tinydots,\rho_u\simiid (\tilde
    l_1/s,\tinydots,\tilde l_n/ s)$,
    \vspace{1mm}
    \STATE\quad\textbf{for} $j=\{1,...,u\}$ \textbf{do}
    \STATE\quad\quad compute
    $l_{\rho_j}=\a_{\rho_j}^\top(\C+\I)^{-1}\a_{\rho_j}$
    \STATE \quad\quad sample
    $z_j\sim\mathrm{Bernoulli}\big(l_{\rho_j}/(2\tilde
    l_{\rho_j})\big)$
    \STATE\quad\textbf{end for}
    \STATE\quad set $\sigma =\{\rho_j:z_j=1\}$, $t=|\sigma|$,
    $\tilde\A_{i,:}=\frac1{\sqrt{r l_{\sigma_i}}}\A_{i,:}$\label{line:poisson2}
    \STATE \quad sample $\textit{Acc}\sim\!
    \text{Bernoulli}\Big(\frac{\ee^{\tilde s}\det(\I+\tilde\A_\sigma^\top\tilde\A_\sigma)}  
{\ee^{t/r}\det(\I+\C)}\Big)$
\vspace{1mm}
    \STATE \textbf{until} $\textit{Acc}=\text{true}$\hfill
    \RETURN $\sigma_{\St}$,\quad where $\St\sim \DPP\big(\tilde\A_\sigma\tilde\A_\sigma^\top\big)$ \label{line:sub}
 \end{algorithmic}
\end{algorithm}

\begin{lemma}\label{l:alg}
After an $O(\nnz(\A)\log(n) + d^4\log(d))$ preprocessing step, we can construct matrix $\C$ and distribution $(\tilde l_1,...,\tilde l_n)$ such that:
\begin{align*}
    (1-\tfrac1{4\sqrt d})\A^\top\A&\preceq \C\preceq (1+\tfrac1{4\sqrt d})\A^\top\A,
    \\
    \frac12\a_i^\top(\C+\I)^{-1}\a_i&\leq \tilde l_i\leq \frac32\a_i^\top(\C+\I)^{-1}\a_i\quad\text{ for all $i$.}
\end{align*}
Moreover, with those inputs, Algorithm~\ref{alg:dpp} returns $S\sim\DPP(\A\A^\top)$ and with probability $1-\delta$ runs in time $O(d^4\log^2(1/\delta))$.
\end{lemma} 
\begin{proof}
The lemma follows nearly identically as the main results of \cite{dpp-intermediate,dpp-sublinear}. One small difference is in the lines \ref{line:poisson1}-\ref{line:poisson2}. Exploiting properties of the Poisson distribution (see proof of Theorem 2 in \cite{alpha-dpp}), these lines can be rewritten as follows:
\begin{algorithmic}
    \STATE \quad compute $\bar s=\sum_i l_i=\tr(\A^\top\A(\C+\I)^{-1})$
    \STATE \quad sample $t \sim
    \mathrm{Poisson}(r\ee^{1/r}\bar s)$,\qquad \label{line:poisson}
    \vspace{1mm}
    \STATE \quad sample $\sigma_1,\tinydots,\sigma_u
    \simiid (l_1/\bar s,\tinydots,\tilde l_n/\bar s)$
\end{algorithmic}
After this modification, we can follow the analysis of \cite{dpp-sublinear} to show the correctness of the algorithm. However, note that the above rewriting requires computing $\bar s=\tr(\A^\top\A(\C+\I)^{-1})$ which takes $O(nd^2)$ time, whereas Algorithm \ref{alg:dpp} avoids that step. The rest of the time complexity analysis for both the preprocessing step and for Algorithm \ref{alg:dpp} is identical to that of \cite{dpp-intermediate} (see version v1 of that paper). 
\end{proof}

Combining Lemmas \ref{l:composition} and \ref{l:alg}, we can now prove Theorem \ref{t:fast-surrogate}, showing that a surrogate row sampling sketch can be constructed in input sparsity time.
\begin{proofof}{Theorem}{\ref{t:fast-surrogate}}
Given matrix $\A$ and a row sampling distribution $p$, we can construct the surrogate row sampling sketch by using Algorithm \ref{alg:dpp} to sample from $\DPP(\frac1\lambda\A\A^\top)$ in time $O(\nnz(\A)\log(n) + d^4\log(d))$, and then using the procedure from Lemma \ref{l:composition}.
\end{proofof}
\section{Matrix Concentration Guarantees for Surrogate Sketches}
\label{sec:concentration_guarantees}

Throughout this section, we will let $C>0$ be a sufficiently large absolute constant. 
Also, we use the notation $\H=\A^\top\A+\lambda\I$ for the Hessian, and we let $\kappa$ be the condition number of $\H$.
Consider the determinantal design $\Det_\mu^\gamma(\A,\lambda)$ where the probability
measure $\mu$ is defined so that
$\mu\big(\{\frac1{\sqrt{p_i}}\e_i^\top\}\big) = p_i$
and $p_i\geq \frac12 \a_i^\top\H^{-1}\a_i/d_\lambda$, for each $i\in[n]$. 

We analyze the concentration properties of the random matrix of the form:
\begin{align*}    
\Z=\H^{-\frac12}(\gamma^{-1}\A^\top\Xb^\top\Xb\A+\lambda\I)\H^{-\frac12}\qquad\text{for}\quad \Xb\sim\Det_\mu^\gamma(\A,\lambda).
\end{align*} 
While standard matrix concentration results apply to sums of independent matrices, recent work of \cite{kyng2018matrix} extended these results to a class of non-i.i.d.~distributions known as Strongly Rayleigh measures. Since determinantal point processes are Strongly Rayleigh, our determinantal designs can also be expressed this way, which allows us to obtain the following guarantee.

\begin{lemma} \label{l:subspace_embedding}
    If $\gamma\geq C d_\lambda\eta^{-2} \log(n/\delta)$ for $\eta\in(0,1)$, then with probability $1-\delta$ we have
        $\|\Z-\I\|\leq \eta$.
\end{lemma} 
We first quote the two matrix concentration results which we use to prove Lemma \ref{l:subspace_embedding}. The more standard result applies to sums of independent random matrices, and can be stated as follows.
\begin{theorem}[\cite{matrix-tail-bounds}]\label{t:tropp}
Consider a finite sequence $\Y_1,\Y_2,...$ of independent random positive semi-definite $d\times d$ matrices that satisfy $\|\Y_i\|\leq R$ and $\mu_{\min}\I\preceq \E[\sum_i\Y_{i}]\preceq \mu_{\max}\I$. Then,
\begin{align*}
\Pr\bigg(\lambda_{\max}\Big(\sum_i\Y_i\Big)\geq(1+\epsilon)\mu_{\max}\bigg)
&\leq d\exp\Big(-\frac{\epsilon^2\mu_{\max}}{3R(1+\epsilon)}\Big)\quad\text{for }\epsilon>0,
\\
\Pr\bigg(\lambda_{\min}\Big(\sum_i\Y_i\Big)\leq(1-\epsilon)\mu_{\min}\bigg)
&\leq d\exp\Big(-\frac{\epsilon^2\mu_{\min}}{2R}\Big)\quad\text{for }\epsilon\in(0,1).
\end{align*}
\end{theorem}
The other matrix concentration result applies to sums of non-independent random matrices, and it can be viewed as a partial extension of the above theorem, except with an additional logarithmic factor. 
\begin{theorem}[\cite{kyng2018matrix}]\label{t:kyng}
Suppose $(\xi_1,...,\xi_n)\in\{0,1\}^n$ is a random vector whose distribution is $k$-homogeneous (i.e., exactly $k$-sparse almost surely) and Strongly Rayleigh. Given $d\times d$ p.s.d.~matrices $\C_1,...,\C_n$ such that $\|\C_i\|\leq R$ and $\|\E[\sum_i\xi_i\C_i]\|\leq \mu$, for any $\epsilon>0$ we have:
\begin{align*}
    \Pr\bigg(\Big\|\sum_i\xi_i\C_i -\E\Big[\sum_i\xi_i\C_i\Big]\Big\|\geq\epsilon\mu\bigg)
    \leq d\exp\Big(-\frac{\epsilon^2\mu}{R(\log k + \epsilon)}\cdot\Theta(1)\Big).
\end{align*}
\end{theorem}
\begin{proofof}{Lemma}{\ref{l:subspace_embedding}}
To perform the analysis, we separate the matrix $\Xb$ into two parts as described in Lemma \ref{l:composition}: the rows coming from $\DPP(\frac1\lambda\A\A^\top)$, and the remaining i.i.d. rows. Denote the former as $\Xb_{\DPP}$ and the latter as $\Xb_{\IID}$. The i.i.d.~part can be analyzed via the  matrix concentration result for independent matrices (Theorem \ref{t:tropp}) by setting 
\[\Y_i=\H^{-\frac12}\big(\tfrac1{\gamma p_{j_i}}\a_{j_i}\a_{j_i}^\top+\tfrac\lambda \gamma\I\big)\H^{-\frac12},\] 
where $M$ is the number of rows in $\Xb_{\IID}$, $\a_j^\top$ denotes the $j$th row of $\A$ and $j_i$ is the row index sampled according to the approximate $\lambda$-ridge leverage score $\{p_j\}$ distribution. Since the number of rows in $\Xb_{\IID}$ is a random variable $M\sim\Poisson(\gamma)$, to apply Theorem \ref{t:tropp} we first condition on $M$. Note that we have 
$\sum_{i=1}^M\Y_i=\H^{-\frac12}(\frac1\gamma\A^\top\Xb_{\IID}^\top\Xb_{\IID}\A+\frac M{\gamma}\lambda\I)\H^{-\frac12}$ and $\E[\sum_{i=1}^m\Y_i\mid M]=\frac M\gamma\I$. Furthermore, using the assumption on the probabilities $p_j$, we have 
$\|\Y_i\|\leq \frac3\gamma$. Thus, Theorem~\ref{t:tropp} implies that (conditioned on $M$):  
\begin{align}
    \Pr\Big(\big\|
\H^{-\frac12}(\gamma^{-1}\A^\top\Xb_{\IID}^\top\Xb_{\IID}\A+\tfrac M\gamma\lambda\I)\H^{-\frac12} - \tfrac M\gamma\I\big\|\geq \eta\cdot \tfrac M\gamma\Big)\leq 2d\exp\big(-\eta^2 M/12\big).\label{eq:iid-tail}
\end{align}
For sufficiently large constant $C$, standard tail bounds for the Poisson distribution (e.g., Theorem 1 in \cite{poisson-tail}) imply that with probability $1-\delta$, we have $|M-\gamma|\leq \eta\gamma$ and, conditioned on this event, \eqref{eq:iid-tail} can be bounded by $\delta$. Putting this together, we conclude that with probability $1-2\delta$:
\begin{align*}
    \|\H^{-\frac12}(\tfrac1\gamma&\A^\top\Xb_{\IID}^\top\Xb_{\IID}\A+\lambda\I)
    \H^{-\frac12} - \I\|
    \\
    &=\|\H^{-\frac12}(\tfrac1\gamma\A^\top\Xb_{\IID}^\top\Xb_{\IID}\A+\tfrac M\gamma\lambda\I)
    \H^{-\frac12} -\tfrac M\gamma\I + (1-\tfrac M\gamma)\lambda\H^{-1} + (\tfrac M\gamma-1)\I \|
    \\
    &\leq\|\H^{-\frac12}(\tfrac1\gamma\A^\top\Xb_{\IID}^\top\Xb_{\IID}\A+\tfrac M\gamma\lambda\I)
    \H^{-\frac12} -\tfrac M\gamma\I\| + 2\,|\tfrac M\gamma-1| \leq 4\eta,
\end{align*}
where we used that $\|\lambda\H^{-1}\|\leq 1$ and $\eta\frac M\gamma\leq\eta(1+\eta)\leq 2\eta$.
 From this it immediately follows that:
\begin{align}
    \Z\succeq \H^{-\frac12}(\gamma^{-1}\A^\top\Xb_{\IID}^\top\Xb_{\IID}\A+\lambda\I)\H^{-\frac12}\succeq (1-4\eta)\I.\label{eq:concent1}
\end{align}
Next, to bound $\Z$ from above we must analyze the contribution of $\Xb_{\DPP}$. Note that the matrix concentration result of \cite{kyng2018matrix} applies to homogeneous Strongly Rayleigh distributions (i.e., where the sample size is fixed) whereas $S\sim \DPP(\frac1\lambda\A\A^\top)$ is non-homogeneous, because the size of $S$ is randomized. However, we can apply a standard transformation known as symmetric homogenization to transform $S$ from a non-homogeneous distribution over subsets of $[n]$ into $\tilde S$ which is an $n$-homogeneous distribution over subsets of $[2n]$, in such a way that $\tilde S\cap [n]$ is distributed identically to $S$ and is Strongly Rayleigh (see Definition 2.12 in \cite{borcea2009negative}).  Now, we can apply Theorem~\ref{t:kyng} with 
$\C_i=\frac1 p_i\H^{-\frac12}\a_i\a_i^\top\H^{-\frac12}$ for $i\in[n]$ and $\C_i=\zero$ for $i>n$, and defining $\xi_i=\one_{[i\in \tilde S]}$. Note that $\sum_i\xi_i\C_i=\H^{-\frac12}\A^\top\Xb_{\DPP}^\top\Xb_{\DPP}\A\H^{-\frac12}$ and $\|\C_i\|\leq 2$. To obtain the expectation of the sum, we use the fact that the marginal probabilities of a DPP are the ridge leverage scores, i.e., $\Pr(i\in S) = \a_i^\top(\A^\top\A+\lambda\I)^{-1}=:l_i(\lambda)$. We obtain that:
\begin{align*}
    \E\Big[\sum_i\xi_i\C_i\Big] = \sum_i\frac{l_i(\lambda)}{p_i}\H^{-\frac12}\a_i\a_i^\top\H^{-\frac12}\preceq 2d_\lambda\H^{-\frac12}\A^\top\A\H^{-\frac12}\preceq 2d_\lambda\I.
\end{align*}
Setting $\epsilon=C\log(n/\delta)/2$ for large enough $C$, Theorem~\ref{t:kyng} implies that with probability $1-\delta$:
\begin{align*}
    \|\H^{-\frac12}\A^\top\Xb_{\DPP}^\top\Xb_{\DPP}\A\H^{-\frac12}\|
    \leq C d_\lambda\log(n/\delta). 
\end{align*}
Note that $\gamma^{-1}C d_\lambda\log(n/\delta)\leq \eta$, so with probability $1-3\delta$:
\begin{align}
    \Z= \tfrac1\gamma\H^{-\frac12}\A^\top\Xb_{\DPP}^\top\Xb_{\DPP}\A\H^{-\frac12} + \H^{-\frac12}(\tfrac1\gamma\A^\top\Xb_{\IID}^\top\Xb_{\IID}\A+\lambda\I)\H^{-\frac12} \preceq (1+5\eta)\I.
    \label{eq:concent2}
\end{align}
Combining \eqref{eq:concent1} and \eqref{eq:concent2} we get $\|\Z-\I\|\leq 5\eta$. Adjusting the constants concludes the proof.
\end{proofof}

 Next, we we use the above matrix concentration result to derive moment bounds for the spectral norm of the random matrix $\Z^{-1}-\E[\Z^{-1}]$. Note that by Lemma \ref{l:least-squares} we have $\E[\Z^{-1}]=\I$. Also, recall that we use $\kappa$ to denote the condition number of $\H$.

\begin{lemma}\label{l:moments}
    If $\gamma\geq C d_\lambda \eta^{-2}p\log(n\kappa/\eta)$ then:
    \begin{align*}
        \E\Big[\big\|\Z^{-1}-\E[\Z^{-1}]\big\|^p\Big]^{\frac1p}\leq \eta.
    \end{align*}
\end{lemma}
\begin{proof}
From Lemma \ref{l:subspace_embedding} we know that if $\gamma\geq C d_\lambda\eta^{-2} \log(n/\delta)$, then with probability $1-\delta$ we have $\|\Z-\I\|\leq\eta$, which also implies that $\|\Z^{-1}-\E[\Z^{-1}]\|\leq\eta$, since $\E[\Z^{-1}]=\I$. Also, note that $\|\Z^{-1}-\E[\Z^{-1}]\|\leq\kappa$ almost surely. It follows that:
\begin{align*}
    \E\Big[\big\|\Z^{-1}-\E[\Z^{-1}]\big\|^p\Big] \leq \eta^p + \delta\kappa^p.
\end{align*}
Letting $\delta = \eta/\kappa^p$, we obtain that $\E[\|\Z^{-1}-\E[\Z^{-1}]\|^p]^{\frac1p}\leq 2\eta$. Adjusting the constants appropriately, we obtain the desired result.
\end{proof}
Using this moment bound, we can show that the average of the inverses of $q$ i.i.d. copies of $\Z$, denoted $\Z_1,...,\Z_q$ exhibit concentration around the mean with high probability.
We first quote a Khintchine/Rosenthal type inequality from \cite{cgt12}, bounding the moments of a sum of random matrices.
\begin{lemma}[\cite{cgt12}]\label{l:rosenthal}
  Suppose that $p\geq 2$ 
  and $r=\max\{p,2\log d\}$. Consider a finite
   sequence $\{\Y_k\}$ of independent, symmetrically random, self-adjoint
   matrices with dimension $d\times d$. Then,
   \begin{align*}
     \E\Big[\big\|\sum\nolimits_k\Y_k\big\|^p\Big]^{\frac 1p}
     \leq\sqrt{\ee r}\,
     \Big\|\sum\nolimits_k\E[\Y_k^2]\Big\|^{\frac 12}+
     2\ee r \,\E\big[\!\max\nolimits_k\|\Y_k\|^p\big]^{\frac 1p}.
   \end{align*}
 \end{lemma}
We are now ready to prove Lemma \ref{l:matrix_concent}. We state here a more precise version of the result. The proof follows by combining Lemmas \ref{l:moments} and \ref{l:rosenthal}, following a standard conversion from moment bounds to high-probability guarantees (our proof is based on the proof of Corollary 13 from \cite{determinantal-averaging}).
\begin{lemma}[restated Lemma \ref{l:matrix_concent}]\label{l:matrix_concent2}
    If $\gamma\geq Cd_\lambda\eta^{-2}\log^4(\frac{n\kappa}{\eta\delta})$ then with probability $1-\delta$, we have:
    \begin{align*}
        \Big\|\frac1q\sum_{k=1}^q(\Z_k^{-1}-\E[\Z_k^{-1}])\Big\|\leq \frac{\eta}{\sqrt q}.
    \end{align*}
\end{lemma}
\begin{proof}
    From Lemma \ref{l:moments}, we know that if $\gamma\geq C d_\lambda \tilde\eta^{-2}p\log(n\kappa/\tilde\eta)$, then for any $2\leq s\leq p$ we have $\E[\|\Z_k^{-1}-\E[\Z_k^{-1}]\|^s]^{\frac1s}\leq \tilde\eta$. Following a standard symmetrization argument, we can write the following:
    \begin{align*}
        \E\Big[\big\|\frac1q\sum_{k=1}^q(\Z_k^{-1}-\E[\Z_k^{-1}])\big\|^p\Big]^{\frac1p}
        \leq 2\cdot\E\Big[\big\|\sum_{k=1}^q\frac{r_k}{q}(\Z_k^{-1}-\E[\Z_k^{-1}])\big\|^p\Big]^{\frac1p},
    \end{align*}
    where $r_k$ are independent Rademacher random variables. We now apply Lemma \ref{l:rosenthal} to the symmetrically random matrices $\Y_k=\frac{r_k}{q}(\Z_k^{-1}-\E[\Z_k^{-1}])$, obtaining:
    \begin{align*}
    \E\Big[\big\|\sum_{k=1}^q\Y_k\big\|^p\Big]^{\frac1p}
    &\leq\sqrt{\ee r}\,
     \Big\|\sum\nolimits_k\E[\Y_k^2]\Big\|^{\frac 12}+
     2\ee r \,\E\big[\!\max\nolimits_k\|\Y_k\|^p\big]^{\frac 1p}
     \\
    &\leq \sqrt{\frac{\ee r}{q}}\cdot\E\big[\|\Z_k^{-1}-\E[\Z_k^{-1}]\|^2\big]^{\frac12}
    +\frac{2\ee r}{\sqrt q}\E\big[\|\Z_k^{-1}-\E[\Z_k^{-1}]\|^p\big]^{\frac1p}
    \leq \frac{C'p}{\sqrt m}\cdot\tilde\eta,
    \end{align*}
    for $p\geq 2\log(d)$. Finally, following the proof of Corollary 13 of \cite{determinantal-averaging}, we use Markov's inequality with $\alpha=\frac{\eta}{\sqrt q}$, $\tilde\eta = \frac{\eta}{4C'p}$ and $p=2\lceil\max\{\log(1/\delta),\log(d)\}\rceil$:
    \begin{align*}
        \Pr\bigg(\Big\|\frac1q\sum_{k=1}^q(\Z_k^{-1}-\E[\Z_k^{-1}])\Big\|\ge\alpha\bigg)
        \leq \alpha^{-p}\cdot
        \E\Big[\big\|\frac1q\sum_{k=1}^q(\Z_k^{-1}-\E[\Z_k^{-1}])\big\|^p\Big]\leq \bigg(\frac{2C'p\tilde\eta}{\alpha\sqrt q}\bigg)^p\leq \delta,
    \end{align*}
    which completes the proof.
\end{proof}
Finally, note that the sketched Hessian matrix $\Hbh_{t}$, defined in Section \ref{sec:unbiased-ns} and used in Lemma \ref{l:matrix_concent}, where recall that $\lambda'=\lambda\cdot(1-\frac{d_\lambda}{m})$, satisfies the following:
\begin{align*}
    \Hbh_{t} = \frac{\lambda}{\lambda'}\A^\top\Sb^\top\Sb\A + \lambda\I 
    = \frac{m}{m-d_\lambda}\A^\top\frac1m\Xb^\top\Xb\A+\lambda\I = \gamma^{-1}\A^\top\Xb^\top\Xb\A+\lambda\I = \H^{\frac12}\Z\H^{\frac12},
\end{align*}
so Lemma \ref{l:matrix_concent} follows immediately from Lemma \ref{l:matrix_concent2} by setting $\eta=\frac1{\sqrt\alpha}$ and assuming that $m\leq n$.

\section{Additional Numerical Results}
\label{sec:additional_numerical}
In this section, we present additional numerical results. 
Figure \ref{fig:intro-debiasing-supplement} complements Figures \ref{fig:intro-debiasing}~and~\ref{fig:regularized_ls_cifar}, demonstrating on two additional datasets that rescaling the local regularizer helps reduce the estimation error of distributed averaging for Gaussian sketch, Rademacher sketch, uniform sampling, and the surrogate sketch.
Figure \ref{fig:regularized_ls_detavg_comparison_supplement} shows the comparison of surrogate sketch, uniform sampling with unweighted averaging, and determinantal averaging, on the Boston housing prices dataset, similarly to Figure \ref{fig:regularized_ls_detavg_comparison} but for different sketch sizes. We observe that the surrogate sketch still empirically outperforms the other methods in the case of different sketch sizes.

\begin{figure} [ht]
\begin{minipage}[b]{0.48\linewidth}
  \centering
  \centerline{\includegraphics[width=\columnwidth]{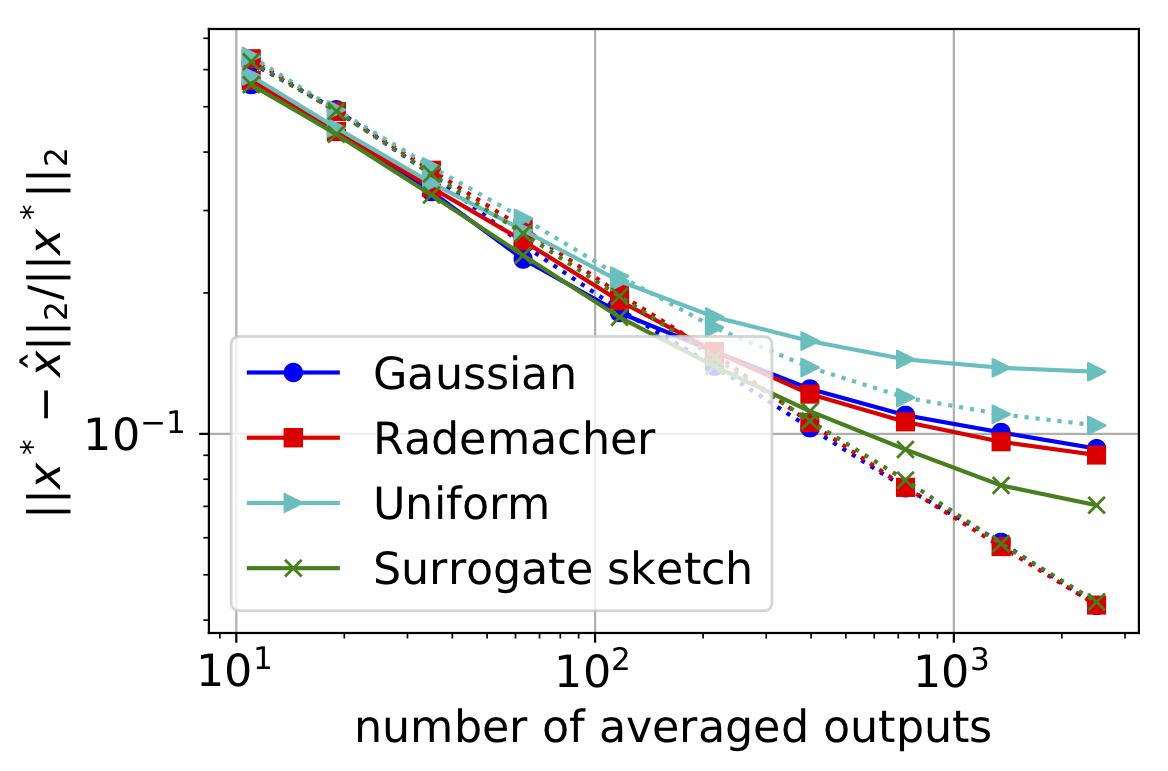}}
  \centerline{(a) diabetes}\medskip
\end{minipage}
\hfill
\begin{minipage}[b]{0.48\linewidth}
  \centering
  \centerline{\includegraphics[width=\columnwidth]{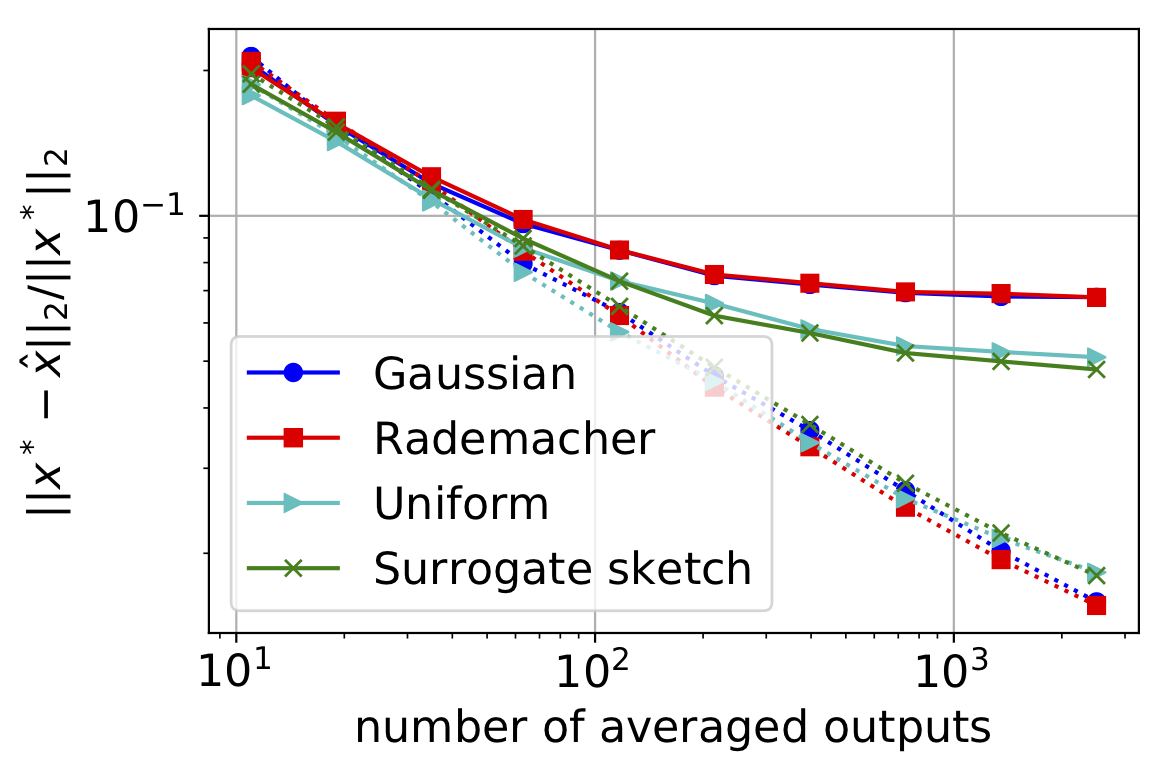}}
  \centerline{(b) monks-1}\medskip
\end{minipage}
\caption{Estimation error against the number of averaged outputs for two different datasets; diabetes and monks-1. The dotted lines show the error when the local regularizer $\lambda^\prime$ is picked using the expression in Theorem \ref{t:intro-unbiased} whereas the straight lines correspond to using the value of the global regularizer as the local regularizer. These plots show the performances of the same sketches as Figure~\ref{fig:intro-debiasing} on two additional datasets. Problem parameters are as follows: Plot a: $n=440$, $d=10$, $m=20$, $\lambda=1$. Plot b: $n=124$, $d=6$, $m=20$, $\lambda=100$.}
\label{fig:intro-debiasing-supplement}
\end{figure}

\begin{figure} [ht]
\begin{minipage}[b]{0.48\linewidth}
  \centering
  \centerline{\includegraphics[width=\columnwidth]{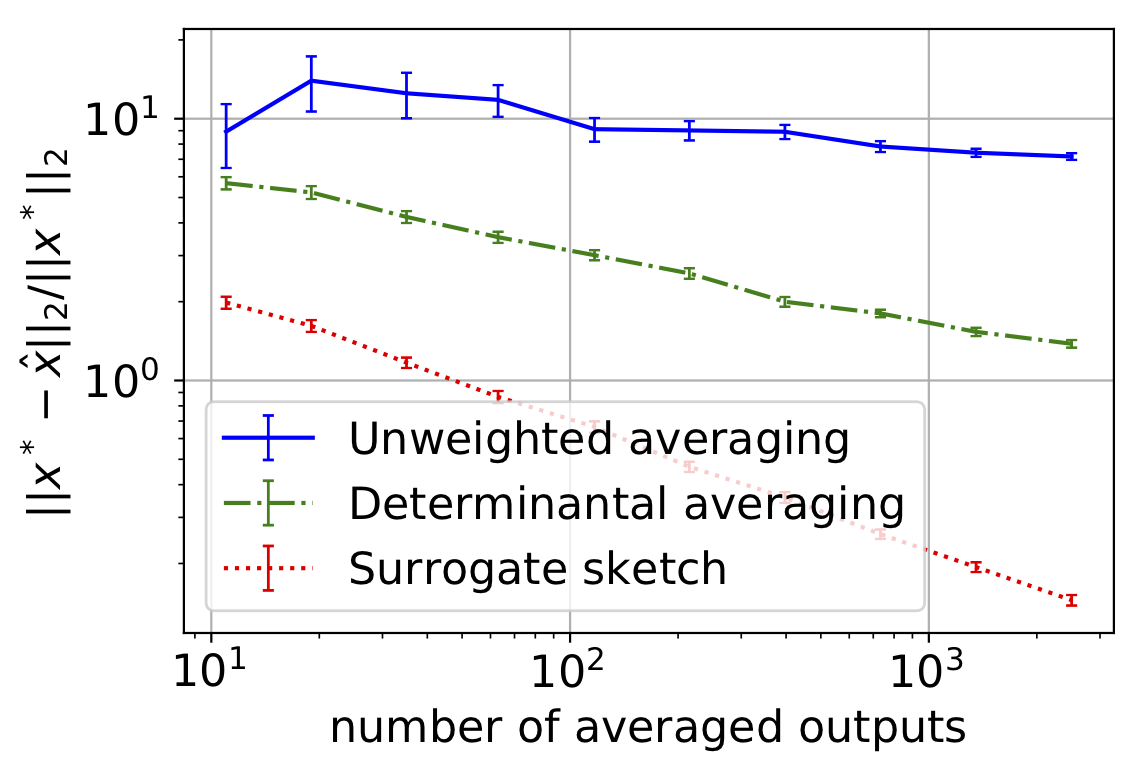}}
  \centerline{(a) $m=20$}\medskip
\end{minipage}
\hfill
\begin{minipage}[b]{0.48\linewidth}
  \centering
  \centerline{\includegraphics[width=\columnwidth]{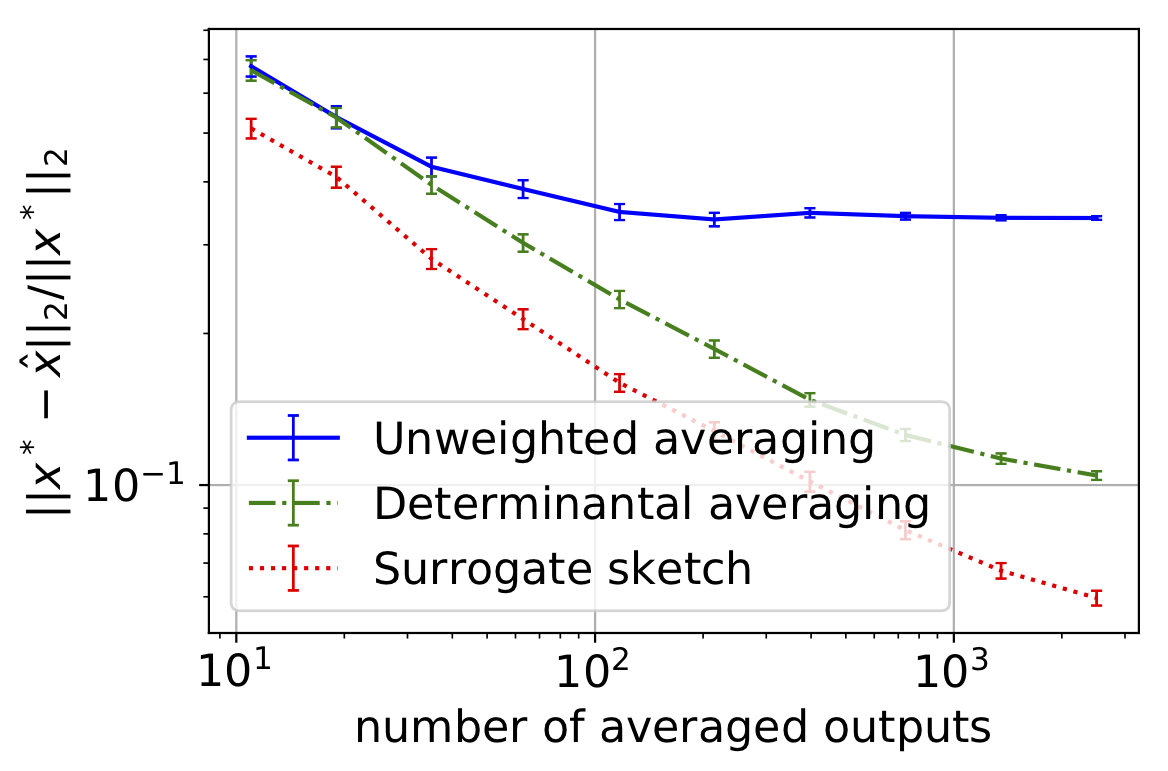}}
  \centerline{(b) $m=100$}\medskip
\end{minipage}
\caption{Estimation error of the \emph{surrogate sketch}, against uniform sampling with \emph{unweighted averaging} \cite{mahoney2018giant} and \emph{determinantal averaging} \cite{determinantal-averaging}. The plots in this figure follow the same setting as Figure \ref{fig:regularized_ls_detavg_comparison} except for the sketch size. In these plots, different sketch sizes are used, which are given in the captions of each plot.}
\label{fig:regularized_ls_detavg_comparison_supplement}
\end{figure}

We next give the exact form of the optimization problem solved in the experiments of Figure \ref{fig:dist_newton_sketch} in the main body of the paper.
The problem solved in Figure \ref{fig:dist_newton_sketch} is a  logistic regression problem with $\ell_2$ regularization:
\begin{align}
     \textrm{minimize}_{\x} \, -\sum_{i=1}^n \left( b_i \log(p_i) + (1-b_i)\log(1-p_i)  \right) + \frac{\lambda}{2} \|\x\|^2,
\end{align}
where $\p \in \mathbb{R}^{n}$ is defined such that $p_i = 1/(1+\exp(-\a_i^\top \x))$. Here $\a_i^\top$ represents the $i$'th row of the data matrix $\A \in \mathbb{R}^{n\times d}$. The output vector is denoted by $\b \in \mathbb{R}^n$. 

In addition to the plots in Figure \ref{fig:dist_newton_sketch}, we include more experimental results for the distributed Newton sketch algorithm. Instead of regularized logistic regression, we consider a different convex optimization problem which has inequality constraints $\|\A\x\|_{\infty} \leq t$ and a quadratic objective $\|\x - \mathbf{c}\|^2$. Problems of this form can be converted into an unconstrained optimization problem using the log-barrier method \cite{bartan2020distributed} to obtain:
\begin{align} \label{log_barrier_opt_prob_2}
    \textrm{minimize}_{\x} \, - \sum_{i=1}^n \log(-\a_i^\top \x+t) -\sum_{i=1}^n \log(\a_i^\top \x+t) + \lambda \|\x\|^2 - 2\lambda \mathbf{c}^\top \x + \lambda \|\mathbf{c}\|^2,
\end{align} 
where $t > 0$, $\mathbf{c} \in \mathbb{R}^d$, and $\a_i$ is the $i$'th row of the data matrix $\A \in \mathbb{R}^{n\times d}$. 
Figure \ref{fig:dist_newton_sketch_logbarrier} shows the normalized error against iterations for the distributed Newton sketch algorithm when it is used to solve the problem \eqref{log_barrier_opt_prob_2}. We note that rescaling the local regularizer as in Theorem \ref{t:intro-unbiased} leads to speedups in convergence for the problem \eqref{log_barrier_opt_prob_2} for Gaussian sketch and surrogate sketch. 

We note that the step sizes for the distributed Newton sketch algorithm in the experiments of both Figure \ref{fig:dist_newton_sketch} and Figure \ref{fig:dist_newton_sketch_logbarrier} have been determined using backtracking line search. The same set of backtracking line search parameters has been used in all of the distributed Newton sketch experiments: Initial step size is set to $\alpha_0=1$, update parameter is $\tau=2$ (where the step size updates are of the form $\alpha \leftarrow \alpha / \tau$), and the control parameter is $c = 0.1$.

\begin{figure}[ht]
\begin{minipage}[b]{0.48\linewidth}
  \centering
  \centerline{\includegraphics[width=\columnwidth]{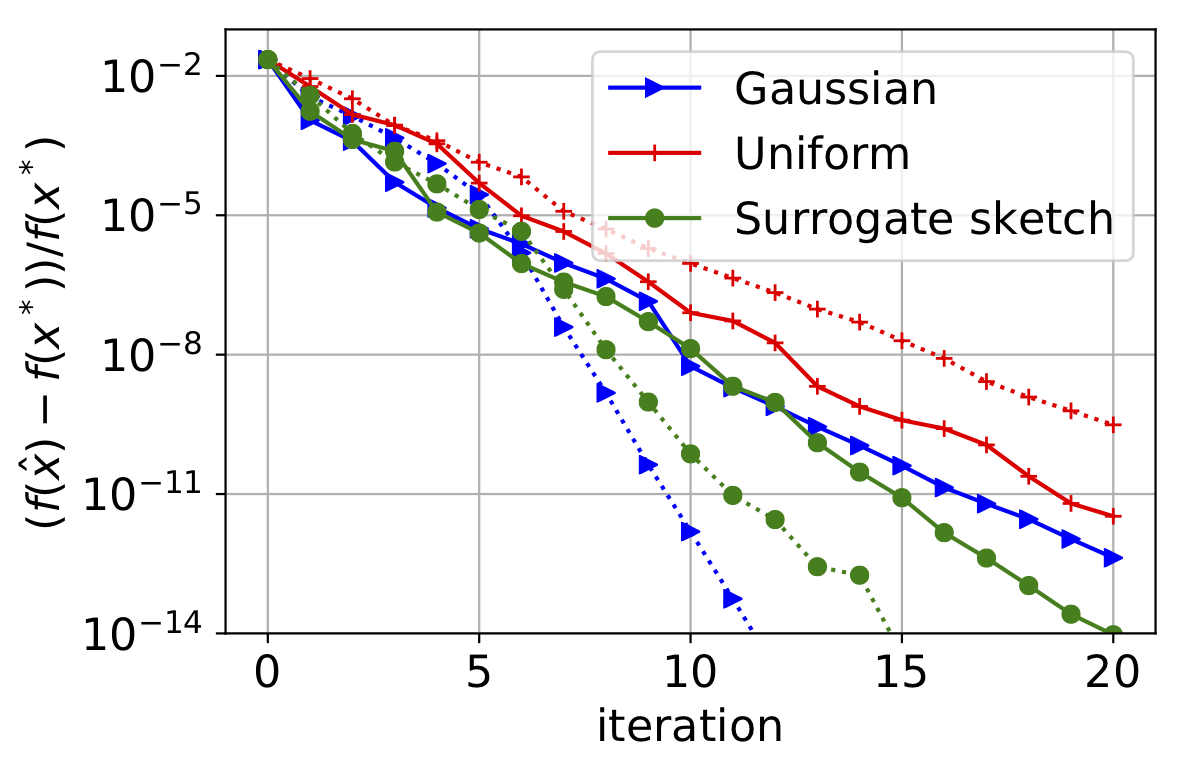}}
  \centerline{(a) $q=20$ workers}\medskip
\end{minipage}
\hfill
\begin{minipage}[b]{0.48\linewidth}
  \centering
  \centerline{\includegraphics[width=\columnwidth]{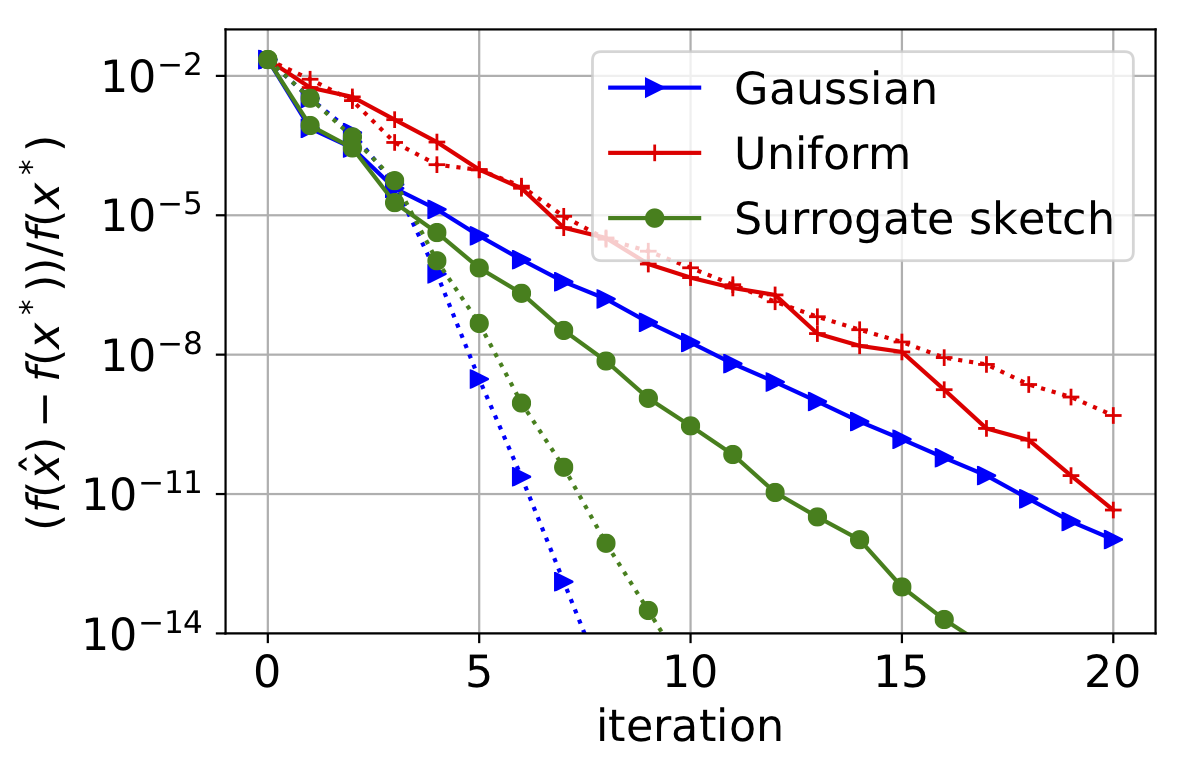}}
  \centerline{(b) $q=100$ workers}\medskip
\end{minipage}
\caption{Distributed Newton sketch algorithm. The problem being solved is an inequality constrained optimization problem, which is transformed into an unconstrained problem using the log-barrier method, as given in \eqref{log_barrier_opt_prob_2}. Data dimensions: $n=500$, $d=50$, $m=100$. $\lambda=10$. The step size for distributed Newton sketch updates has been determined via backtracking line search with parameters $\tau = 2$, $c = 0.1$, $a_0 = 1$. The dotted lines show the error when the local regularizer $\lambda^\prime$ is picked using the expression in Theorem \ref{t:intro-unbiased} whereas the straight lines correspond to using the value of the global regularizer as the local regularizer.}
\label{fig:dist_newton_sketch_logbarrier}
\end{figure}

Figure \ref{fig:regularized_ls_lm2_sweep} shows for various sketches that choosing the local regularizer as stated in Theorem \ref{t:intro-unbiased} in fact leads to the lowest estimation error possible that one could hope to achieve by optimizing over the local regularization parameter. The dotted blue lines show the estimation error if we were to choose the local regularizer to be values from $1$ to $10$. The straight red line shows the error when the local regularizer $\lambda^\prime$ is picked using the expression in Theorem \ref{t:intro-unbiased}. We observe that in the regime where the number of workers is high, the lowest estimation error is achieved by the red line. We also see that this empirical observation is true for not only the surrogate sketch but also Gaussian sketch and uniform sampling.

\begin{figure}[ht]
\begin{minipage}[b]{0.32\linewidth}
  \centering
  \centerline{\includegraphics[width=\columnwidth]{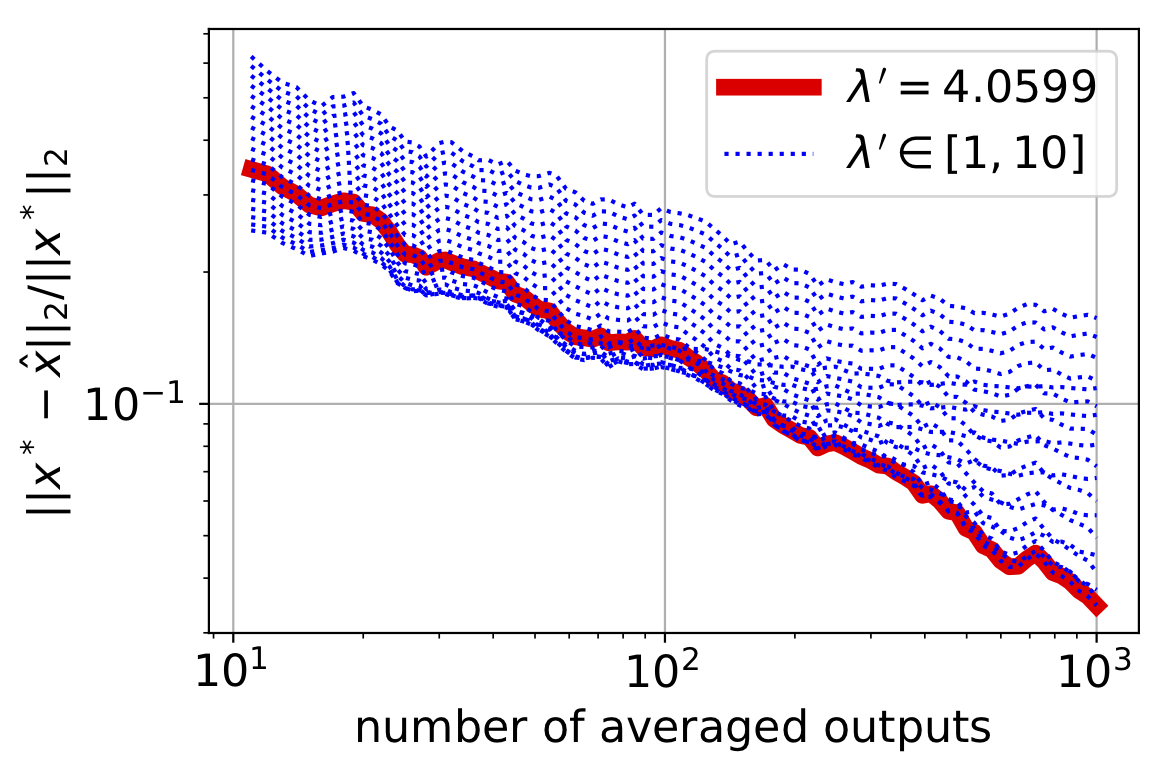}}
  \centerline{(a) Gaussian}\medskip
\end{minipage}
\hfill
\begin{minipage}[b]{0.34\linewidth}
  \centering
  \centerline{\includegraphics[width=\columnwidth]{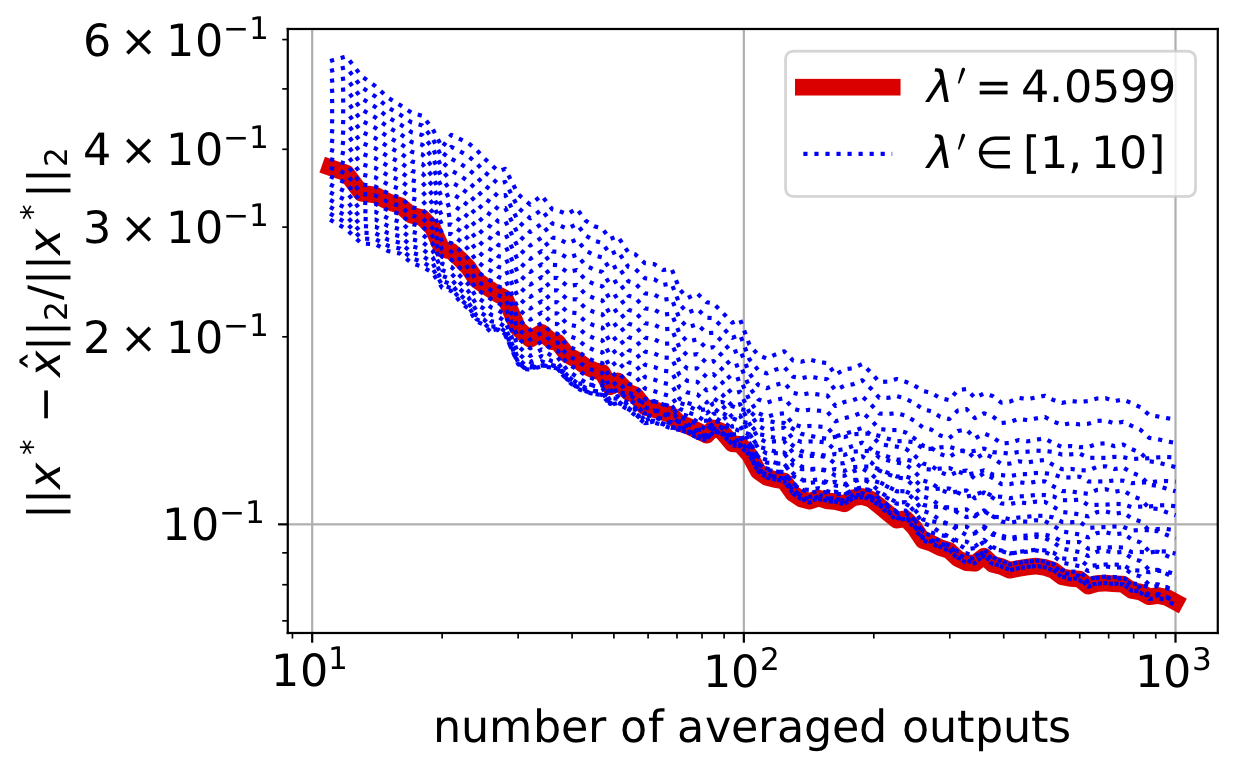}}
  \centerline{(b) Uniform}\medskip
\end{minipage}
\hfill
\begin{minipage}[b]{0.32\linewidth}
  \centering
  \centerline{\includegraphics[width=\columnwidth]{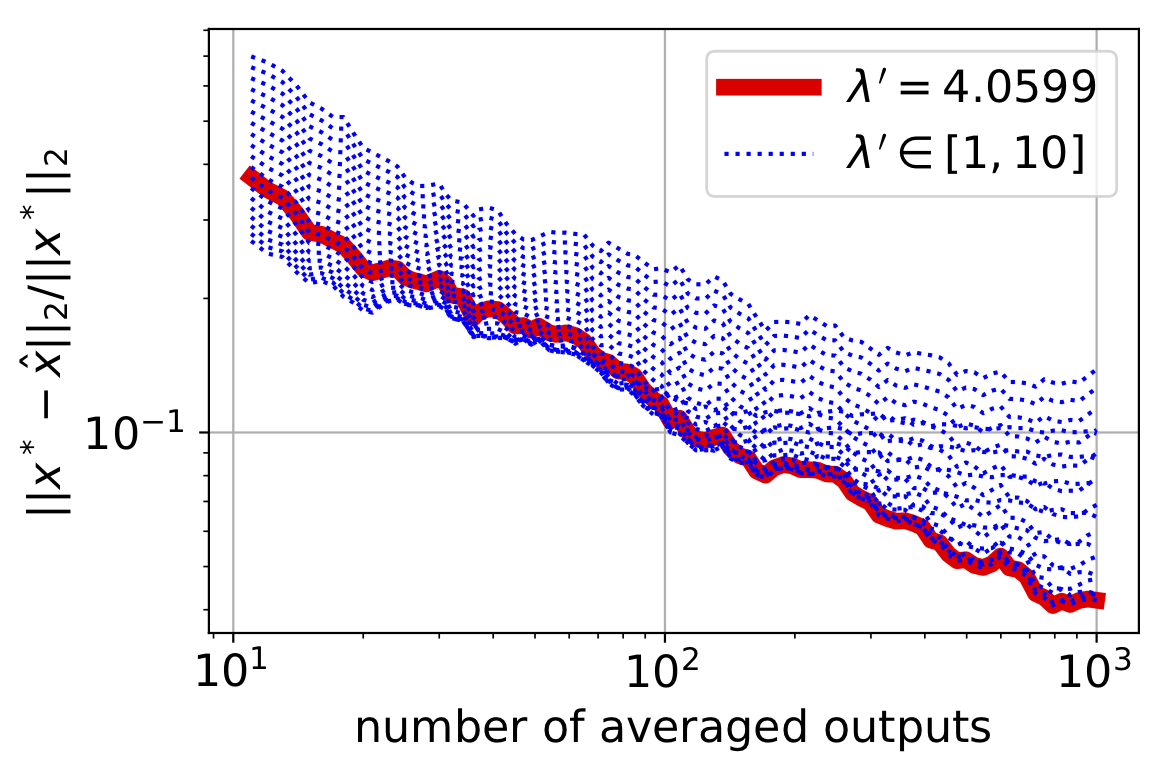}}
  \centerline{(c) Surrogate sketch}\medskip
\end{minipage}
\caption{Estimation error against the number of workers for various sketches for the Boston housing prices dataset. The red line shows the error for the debiased version according to Theorem \ref{t:intro-unbiased} and the dotted blue lines correspond to using different local regularization parameter values from $[1, 10]$. All of the curves are averaged over $25$ independent trials. Sketch size is $m=20$ and the global regularization parameter is $\lambda=10$. }
\label{fig:regularized_ls_lm2_sweep}
\end{figure}

\end{document}